\pdfoutput=1

\documentclass[twoside]{article}

\usepackage[accepted]{aistats2024}
\usepackage{bm}
\usepackage{xcolor}
\usepackage{amsmath}
\usepackage{amsthm}
\usepackage{amsfonts}       
\usepackage{algorithm}
\usepackage{algorithmic}
\usepackage{subcaption}
\usepackage{graphicx}
\graphicspath{ {./img/} }

\usepackage{Sty/mcr}

\newtheorem{lemma}{Lemma}
\newtheorem{remark}{Remark}
\newtheorem{example}{Example}

\DeclareMathOperator*{\argmin}{arg\,min}

\begin{document}

%
\runningtitle{CAD-DA: Controllable Anomaly Detection after Domain Adaptation by Statistical Inference}

%

\twocolumn[

\aistatstitle{CAD-DA: Controllable Anomaly Detection after \\ Domain Adaptation by Statistical Inference}

\aistatsauthor{ Vo Nguyen Le Duy \And Hsuan-Tien Lin \And  Ichiro Takeuchi }

\aistatsaddress{ RIKEN \And  National Taiwan University \And Nagoya University/RIKEN } ]

\begin{abstract}

We propose a novel statistical method for testing the results of anomaly detection (AD) under domain adaptation (DA), which we call CAD-DA---controllable AD under DA. The distinct advantage of the CAD-DA lies in its ability to control the probability of misidentifying anomalies under a pre-specified level $\alpha$ (e.g., 0.05). The challenge within this DA setting is the necessity to account for the influence of DA to ensure the validity of the inference results. We overcome the challenge by leveraging the concept of Selective Inference to handle the impact of DA. To our knowledge, this is the first work capable of conducting a valid statistical inference within the context of DA. We evaluate the performance of the CAD-DA method on both synthetic and real-world datasets.

\end{abstract}


\section{Introduction} \label{sec:intro}


\begin{figure*}[!t]
\centering
\includegraphics[width=.85\textwidth]{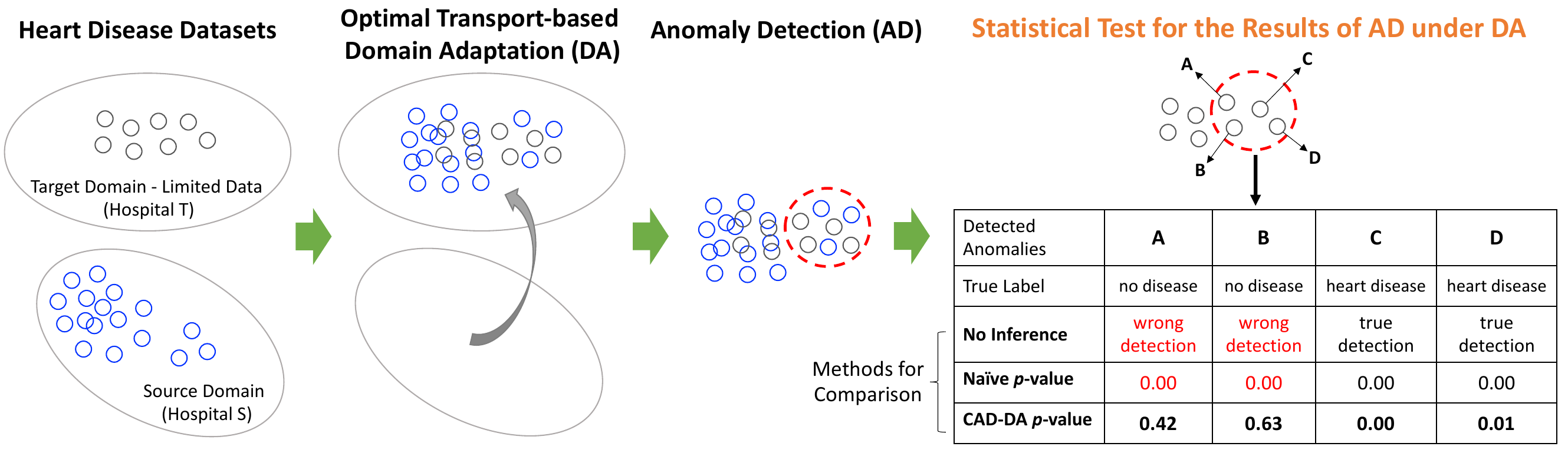}
\caption{Illustration of the proposed method.
Conducting AD-DA without inference results wrong anomalies (\textbf{A}, \textbf{B}).
The naive $p$-values are small even for falsely detected anomalies.
The proposed CAD-DA can identify both false positive (FP) and true positive (TP) detections, i.e., large $p$-values for FPs and small $p$-values for TPs.
}
\label{fig:illustration}
\vspace{-5pt}
\end{figure*}

%

Anomaly detection (AD) is a fundamental problem in machine learning and statistics, as the vast body of literature surveyed by \cite{aggarwal2017outlier} suggests.
%
%
%
%
%
The goal of AD is to identify rare and unusual observations that deviate significantly from the norm within a given dataset.
AD plays a critical role in several applications and has been widely applied in many areas such as medical \cite{wong2002rule, aggarwal2005abnormality}, fraud detection \cite{pourhabibi2020fraud}, and damage detection \cite{avci2021review, du2020damage}.

In numerous real-world scenarios, the availability of limited data can lead to poor performance of AD. 
To overcome this challenge, a strategy of increasing the sample size by transferring data points from a readily accessible source domain to the present target domain can be employed. 
This type of problem is known as Domain Adaptation (DA).
By leveraging data-rich source domains to bolster the data pool in the target domain, DA aims to enhance the efficacy of AD in scenarios where limited data hamper their effectiveness.

A critical concern arises regarding to the possibility of erroneous detection.
%
The AD could misidentify certain observations as anomalies, even though they are actually normal. 
These errors are commonly referred to as \emph{false positives},
%
which can cause serious consequences in high-stake decision making.
%
Especially, when conducting AD under DA, there is an increased risk of misclassifying normal instances as anomalies due to potential DA errors.
%
%
For instance, in the medical field, some unhealthy individuals transformed from the source domain may become closely similar to healthy individuals in the target domain.
Thus, we mistakenly classify a healthy individual as unhealthy, and we could inadvertently administer drugs that may harm their well-being. 
%
%
%
Hence, there is a critical need of an inference method  for controlling the false positive rate (FPR).

In AD-DA, controlling the false negative rate (FNR) is also important.
%
In the literature of statistics, a common procedure is to initially control the FPR at a specified level $\alpha$, e.g., 0.05, while concurrently seeking to minimize the FNR, i.e., maximizing the true positive rate (${\rm TPR} = 1 - {\rm FNR}$)
by empirical evidences.
In this paper, we also follow this established practice.
We propose a method to theoretically control the probability of misidentifying a normal observation as an anomaly while endeavoring to minimize the probability of misidentifying an anomaly as normal.

To our knowledge, none of the existing method can control the FPR of AD under the context of DA.
The main challenge is that, without accounting for the influence of DA, the FPR can not be properly controlled.
%
%
The authors of \cite{tsukurimichi2022conditional} proposed a method for testing the anomalies when they are detected by a class of robust regression methods \cite{huber1973robust, andrews1974robust, zaman2001econometric, rousseeuw2005robust, maronna2019robust}.
%
However, this method supposes that the data comes from the same distribution, and it is invalid in the scenario that a distribution shift occurs and DA must be applied.

Our idea is to leverage \emph{Selective Inference (SI)} \cite{lee2016exact} for resolving this challenge.
However, directly incorporating SI in our setting is still non-trivial because SI is highly \emph{problem-specific}.
Therefore, we need to carefully examine the selection strategy of the algorithm in the context of AD after DA.
Moreover, if we naively leverage the ideas of existing SI methods \cite{lee2016exact, duy2021exact}, the power of the test is significantly low, i.e, the FNR is high.
Therefore, we need to introduce an approach to minimize the FNR while properly controlling the FPR.
We would like to note that we start this new research direction with the Optimal Transport (OT)-based DA \cite{flamary2016optimal}, which is recently promising and popular in the OT community. The detailed discussions on future extensions to other types of DA are provided in \S \ref{sec:discussion}.


\paragraph{Contributions.} Our contributions are as follows:

$\bullet$ We mathematically formulate the problem of testing AD results under DA and introduce \emph{CAD-DA}, a novel method to conduct the statistical test with controllable FPR. 
Compared to the literature on controllable AD, CAD-DA presents a unique challenge in addressing the effect of DA to ensure the validity of FPR control.

$\bullet$ We overcome the challenge by leveraging the SI framework to handle the influence of DA.
We carefully examine the selection strategy for OT-based DA, whose operations can be characterized by linear/quadratic inequalities, and prove that achieving controllable AD-DA is indeed possible.
Furthermore, we introduce a more strategic approach to enhance the TPR.
To our knowledge, this is the first work capable of conducting valid inference within the context of DA.

%

$\bullet$ We conduct experiments on both synthetic and real-world datasets to support our theoretical results, showcasing superior performance of the CAD-DA.


\begin{example}
To show the importance of the proposed method, we provide an example presented in Fig. \ref{fig:illustration}.
Our goal is to detect patients in Hospital T with heart disease, treated as anomalies, in a scenario where the number of patients is limited.
Here, the source domain consists of patients in Hospital S, while the target domain comprises patients in Hospital T.
We employ the OT-based DA approach to transfer the data from the source domain to the target domain.
Subsequently, we apply an AD algorithm, i.e., mean absolute deviation.
The AD algorithm erroneously identified two healthy individuals as anomalies. 
To address this issue, we conducted an additional inference step using the proposed $p$-values, allowing us to identify both true positive and false positive detections.
Furthermore, we repeated the experiments $N$ times and the FPR results are shown in Tab. \ref{tbl:example_intro}.
With the proposed method, we were able to control the FPR under $\alpha = 0.05$, which other competitors were unable to achieve. 
\end{example}

\begin{table}[t]
\renewcommand{\arraystretch}{1.2}
\centering
\caption{The importance of the proposed method lies in its ability to control the False Positive Rate (FPR).}
\vspace{-5pt}
\begin{tabular}{ |l|c|c|c| } 
  \hline
  & \textbf{No Inference} & \textbf{Naive} & \textbf{CAD-DA} \\
  \hline
  $N = 120$ & FPR = 1.0 & 0.6 & \textbf{0.04} \\
   \hline
  $N = 240$ & FPR = 1.0 & 0.7 & \textbf{0.05} \\
  \hline
\end{tabular}
\label{tbl:example_intro}
\vspace{-10pt}
\end{table}

\textbf{Related works.}
Although there exists an extensive body of literature on AD methods \cite{aggarwal2017outlier}, there has been a limited exploration of applying the hypothesis testing framework to assess the results of AD. 
For instance, in \cite{srivastava1998outliers} and \cite{pan1995multiple}, the likelihood ratio test has been discussed to determine if an individual data point is an anomaly, employing the mean-shift model.
%
However, these classical anomaly inference methods hold validity exclusively when the target anomalies are pre-determined in advance. 
When we employ these classical techniques on the anomalies detected by an AD algorithm, they become invalid in the sense that the FPR cannot be controlled at the desired level.

In order to control the FPR when using classical methods, the use of multiple testing correction is essential.
The most popular technique is Bonferroni correction, in which the correction factor scales exponentially with the number of instances, denoted as $n$. 
Specifically, it grows to a value of $2^n$.
However, this correction factor can become prohibitively large unless $n$ fairly small, which leads to overly conservative statistical inference.

In recent years, SI has emerged as a promising approach to resolve the invalidity of the traditional inference method without the need of employing conservative multiple testing correction. 
Specifically, instead of considering the exponentially increasing value of $2^n$, we  consider the correction factor of $1$ by conducting the \emph{conditional inference} conditioning on the single set of detected anomalies. 
This is the basic concept of the conditional SI introduced in the seminal work of  \cite{lee2016exact}.

The seminal paper has not only laid the foundation for research on SI for feature selection~\cite{loftus2014significance, fithian2015selective, tibshirani2016exact, yang2016selective, hyun2018exact, sugiyama2021more, fithian2014optimal, duy2021more} but has also spurred the development of SI for more complex supervised learning algorithms, such as boosting~\cite{rugamer2020inference}, decision trees~\cite{neufeld2022tree}, kernel methods~\cite{yamada2018post}, higher-order interaction models~\cite{suzumura2017selective,das2021fast} and deep neural networks~\cite{duy2022quantifying, miwa2023valid}.
Moreover, SI is also valuable for unsupervised learning problems, such as change point detection~\cite{umezu2017selective, hyun2018post, duy2020computing, sugiyama2021valid, jewell2022testing}, clustering~\cite{lee2015evaluating, inoue2017post, gao2022selective, chen2022selective}, and segmentation~\cite{tanizaki2020computing, duy2022quantifying}.
Furthermore, SI can be applied to statistical inference on  the DTW distance~\cite{duy2022exact} and the Wasserstein distance~\cite{duy2021exact}.

The studies most related to this paper are \cite{chen2019valid, tsukurimichi2022conditional}.
The authors of \cite{chen2019valid} introduced a  method for testing the features of a linear model after removing anomalies.
Although their work did not directly address the problem of testing the anomalies, it was the  inspiration for 
\cite{tsukurimichi2022conditional}.
The contribution of \cite{tsukurimichi2022conditional} centered on introducing a SI approach for testing the anomalies identified by a class of robust regression methods.
However, a notable limitation of \cite{chen2019valid} and \cite{tsukurimichi2022conditional} is their assumption that the data comes from the same distribution.
Therefore, when applied in the context of DA, their methods loses its validity and applicability, making them unsuitable for scenarios involving DA.
Besides, their primary focus revolves around the context of linear regression, which differs from the setting we consider in this paper. 

\section{Problem Setup} \label{sec:problem_setup}

Let us consider two random vectors
\begin{align*} 
	\bm X^s &= (X^s_1, ..., X^s_{n_s})^\top = \bm \mu^s  + \bm \veps^s, \quad \bm \veps^s \sim \NN(\bm 0, \Sigma^s), 
	\\ 
	\bm X^t &= (X^t_1, ..., X^t_{n_t})^\top = \bm \mu^t  + \bm \veps^t, \quad \bm \veps^t \sim \NN(\bm 0, \Sigma^t), 
\end{align*}
where $n_s$ and $n_t$ are the number of instances in the source and target domains, $\bm \mu^s$ and $\bm \mu^t$ are unknown signals, $\bm \veps^s$ and $\bm \veps^t$ are the Gaussian noise vectors with the covariance matrices $\Sigma^s$ and $\Sigma^t$ assumed to be known or estimable from independent data. 
We assume that the number of instances in the target domain is limited, i.e., $n_t$ is much smaller than $n_s$.
The goal is to statistically test the results of AD after DA.

\paragraph{Optimal Transport (OT)-based DA \cite{flamary2016optimal}.} 
Let us define the cost matrix as 
\begin{align*} 
	C(\bm X^s, \bm X^t) 
	& = \big[(X_i^s - X_j^t)^2 \big]_{ij} \in \RR^{n_s \times n_t}.
\end{align*}
The OT problem between the source and target domain is then defined as 
\begin{align} \label{eq:ot_problem}
		\hat{T} = \argmin \limits_{T \geq 0} & ~ \langle T, C(\bm X^s, \bm X^t)\rangle \\ 
		\text{s.t.} ~~& ~ T \bm{1}_{n_t} = \bm 1_{n_s}/{n_s},  T^\top \bm{1}_{n_s} = \bm 1_{n_t}/{n_t} \nonumber,
\end{align}
where $\bm{1}_n \in \RR^n$ is the vector whose elements are set to $1$.
After obtaining the optimal transportation matrix $\hat{T}$, source instances are transported in the target domain.
The transformation $\tilde{\bm X}^s$ of $\bm X^s$ is defined as:
\begin{align*}
	\tilde{\bm X}^s 
		= n_s \hat{T} \bm X^t.
\end{align*} 
More details are provided in Sec 3.3 of \cite{flamary2016optimal}. 

\paragraph{Anomaly detection.} After transforming the data from source domain to the target domain, we apply an AD algorithm $\cA$ on $\big \{ \tilde{\bm X}^s, \bm X^t \big \}$ to obtain a set $\cO$ of indices of anomalies in the target domain:
\begin{align*}
		\cA: \left \{ \tilde{\bm X}^s, \bm X^t \right \} 
		\mapsto
		\cO \in [n_t].
\end{align*}
In this paper, we used the Median Absolute Deviation (MAD) as an example of the AD algorithm. 
%
%
Our proposed CAD-DA framework is not specialized for a specific AD algorithm but can also be applied to other AD algorithms (see \S \ref{subsec:identification_cZ} for more details).

\paragraph{Statistical inference and decision making with a $p$-value.} To statistically quantifying the significance of the detected anomalies, we consider the statistical test on the following null and alternative hypotheses:
\begin{align*} 
	{\rm H}_{0, j}: \mu^t_j = \bar{\bm \mu}^t_{- \cO}
	\quad
	\text{vs.}
	\quad 
	{\rm H}_{1, j}: \mu^t_j \neq \bar{\bm \mu}^t_{- \cO}, \quad 
	\forall j \in \cO,
\end{align*}
where 
\[
	\bar{\bm \mu}^t_{- \cO} = 
	\frac{1}{n_t - |\cO|} \sum \limits_{\ell \in [n_t] \setminus \cO}
	\mu^t_\ell.
\]
In other words, our goal is to test if each of the detected anomalies $j \in \cO$ is truly deviated from the remaining data points after removing the set of anomalies $\cO$.

To test the hypotheses, the test statistic is defined as:
\begin{align}\label{eq:test_statistic}
	 T_j 
	 = X_j^t - \bar{\bm X}^t_{- \cO} 
	 = \bm \eta_j^\top {\bm X^s \choose \bm X^t }, 
\end{align}
where $\bm \eta_j$ is the direction of the test statistic: 
\begin{align} \label{eq:eta_j}
\bm \eta_j = 
\begin{pmatrix}
	\bm 0^{s} \\ 
	\bm e^{t}_j - \frac{1}{n_t - |\cO|}
	\bm e^{t}_{-\cO}
\end{pmatrix},
\end{align}
$\bm 0^{s} \in \RR^{n_s}$ represents a vector where all entries are set to 0, 
$\bm e^{t}_j \in \RR^{n_t}$ is a vector in which the $j^{\rm th}$ entry is set to $1$, and $0$ otherwise,
$\bm e^{t}_{-\cO} \in \RR^{n_t}$ is a vector in which the $j^{\rm th}$ entry is set to $0$ if $j \in \cO$, and $1$ otherwise.

After obtaining the test statistic in \eq{eq:test_statistic}, we need to compute a $p$-value.
Given a significance level $\alpha \in [0, 1]$, e.g., 0.05, we reject the null hypothesis ${\rm H}_{0, j}$ and assert that $X_j^t$ is an anomaly if the $p$-value $ \leq \alpha$.
Conversely, if the $p$-value $ > \alpha$, we infer that there is insufficient evidence to conclude that $X_j^t$ is an anomaly.

\paragraph{Challenge of computing a valid $p$-value.}
The traditional (naive) $p$-value, which does not properly consider the effect of DA and AD, is defined as:
\begin{align*}
	p^{\rm naive}_j = 
	\mathbb{P}_{\rm H_{0, j}} 
	\Bigg ( 
		\left | \bm \eta_j^\top {\bm X^s \choose \bm X^t } \right |
		\geq 
		\left | \bm \eta_j^\top {\bm X^s_{\rm obs} \choose \bm X^t_{\rm obs} } \right |
	\Bigg ), 
\end{align*}
where $\bm X^s_{\rm obs}$ and $\bm X^t_{\rm obs}$ are the observations of the random vectors $\bm X^s$ and $\bm X^t$, respectively.
If the vector $\bm \eta_j$ is independent of the DA and AD algorithms, the naive $p$-value is valid in the sense that 
\begin{align} \label{eq:valid_p_value}
	\mathbb{P} \Big (
	\underbrace{p_j^{\rm naive} \leq \alpha \mid {\rm H}_{0, j} \text{ is true }}_{\text{a false positive}}
	\Big) = \alpha, ~~ \forall \alpha \in [0, 1],
\end{align} 
i.e., the probability of obtaining a false positive is controlled under a certain level of guarantee.
However, in our setting, the vector $\bm \eta_j$ actually depends on the DA and AD.
The property of a valid $p$-value in \eq{eq:valid_p_value} is no longer satisfied.
Hence, the naive $p$-value is \emph{invalid}.

\section{Proposed CAD-DA Method} \label{sec:method}


In this section, we introduce the technical details of computing the valid $p$-value in CAD-DA.

\subsection{The valid $p$-value in CAD-DA} \label{subsec:valid_p_value}

To calculate the valid $p$-value, we need to derive the sampling distribution of the test statistic in \eq{eq:test_statistic}.
To achieve the task, we utilize the concept of conditional SI, i.e., we consider the sampling distribution of the test statistic  \emph{conditional} on the AD results after DA:
\begin{align} \label{eq:conditional_distribution}
	\mathbb{P} \Bigg ( 
	\bm \eta_j^\top {\bm X^s \choose \bm X^t }
	~
	\Big |
	~ 
	\cO_{\bm X^s, \bm X^t}
	=
	\cO_{\rm obs}
	\Bigg ),
\end{align}
where $\cO_{\bm X^s, \bm X^t}$ is the results of AD after DA \emph{for any random} $\bm X^s$ and $\bm X^t$, and $\cO_{\rm obs} = \cO_{\bm X^s_{\rm obs}, \bm X^t_{\rm obs}}$.

Next, we introduce the \emph{selective $p$-value} defined as
\begin{align} \label{eq:selective_p}
	p^{\rm sel}_j = 
	\mathbb{P}_{\rm H_{0, j}} 
	\Bigg ( 
		\left | \bm \eta_j^\top {\bm X^s \choose \bm X^t } \right |
		\geq 
		\left | \bm \eta_j^\top {\bm X^s_{\rm obs} \choose \bm X^t_{\rm obs} } \right |
		~
		\Bigg | 
		~
		\cE
	\Bigg ), 
\end{align}
where the conditioning event $\cE$ is defined as
\begin{align} \label{eq:conditioning_event}
\cE = \Big \{ 
	\cO_{\bm X^s, \bm X^t}
	=
	\cO_{\rm obs},
	\cQ_{\bm X^s, \bm X^t}
	=
	\cQ_{\rm obs}
\Big \}. 
\end{align}
The $\cQ_{\bm X^s, \bm X^t}$ is the \emph{nuisance component} defined as 
\begin{align} \label{eq:q_and_b}
	\cQ_{\bm X^s, \bm X^t} = 
	\left ( 
	I_{n_s + n_t} - 
	\bm b
	\bm \eta_j^\top \right ) 
	{\bm X^s \choose \bm X^t},
\end{align}
where 
$
	\bm b = \frac{\Sigma \bm \eta_j}
	{\bm \eta_j^\top \Sigma \bm \eta_j}
$
and 
$
\Sigma = 
\begin{pmatrix}
	\Sigma^s & 0 \\ 
	0 & \Sigma^t
\end{pmatrix}.
$

\begin{remark}
The nuisance component $\cQ_{\bm X^s, \bm X^t}$ corresponds to
the component $\bm z$ in the seminal paper of \cite{lee2016exact} (see Sec. 5, Eq. (5.2), and Theorem 5.2)).
We note that additionally conditioning on $\cQ_{\bm X^s, \bm X^t}$, which is required for technical purpose, is a standard approach in the SI literature and it is used in almost all the SI-related works that we cited.
%
\end{remark}

\begin{lemma} \label{lemma:valid_selective_p}
The selective $p$-value proposed in \eq{eq:selective_p} satisfies the property of a valid $p$-value:
\begin{align*}
	\mathbb{P}_{\rm H_{0, j}}  \Big (
	p_j^{\rm sel} \leq \alpha
	\Big) = \alpha, ~~ \forall \alpha \in [0, 1].
\end{align*} 
\end{lemma}

\begin{proof}
The proof is deferred to Appendix \ref{appx:proof_valid_selective_p}.
\end{proof}

Lemma \ref{lemma:valid_selective_p} indicates that, by using the selective $p$-value, the FPR is theoretically controlled at the significance level $\alpha$.
To compute the selective $p$-value, we need to identify the conditioning event $\cE$ in \eq{eq:conditioning_event} whose characterization will be introduced in the next section.


\begin{figure*}[!t]
\begin{minipage}{0.71\textwidth}
\centering
\includegraphics[width=\linewidth]{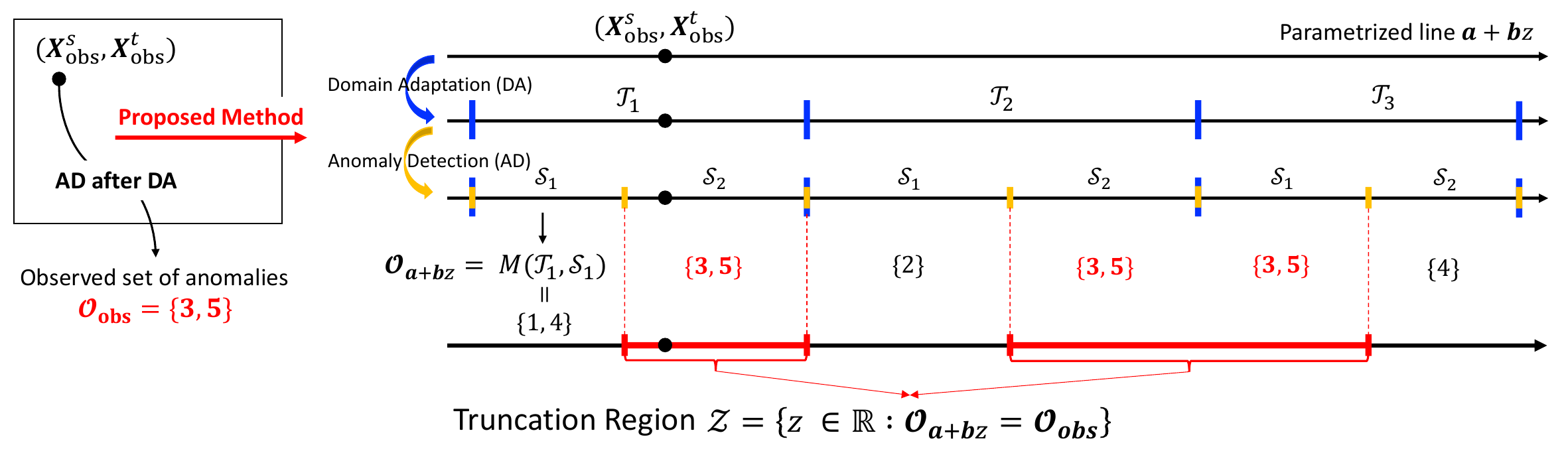} 
\caption{
\footnotesize 
A schematic illustration of the proposed method.
By applying AD after DA to the observed data, we obtain a set of anomalies.
Then, we parametrize the data with a scalar parameter $z$ in the dimension of the test statistic to identify the truncation region $\cZ$ whose data have the \emph{same} result of anomaly detection as the observed data.
Finally, the valid inference is conducted conditional on $\cZ$.
We utilize the concept of ``divide-and-conquer'' and introduce a hierarchical line search method for efficiently characterizing the truncation region $\cZ$.}
\label{fig:schematic_illustration}
\end{minipage}
\hspace{2pt}
\begin{minipage}{0.28\textwidth}
\centering
\includegraphics[width=.95\linewidth]{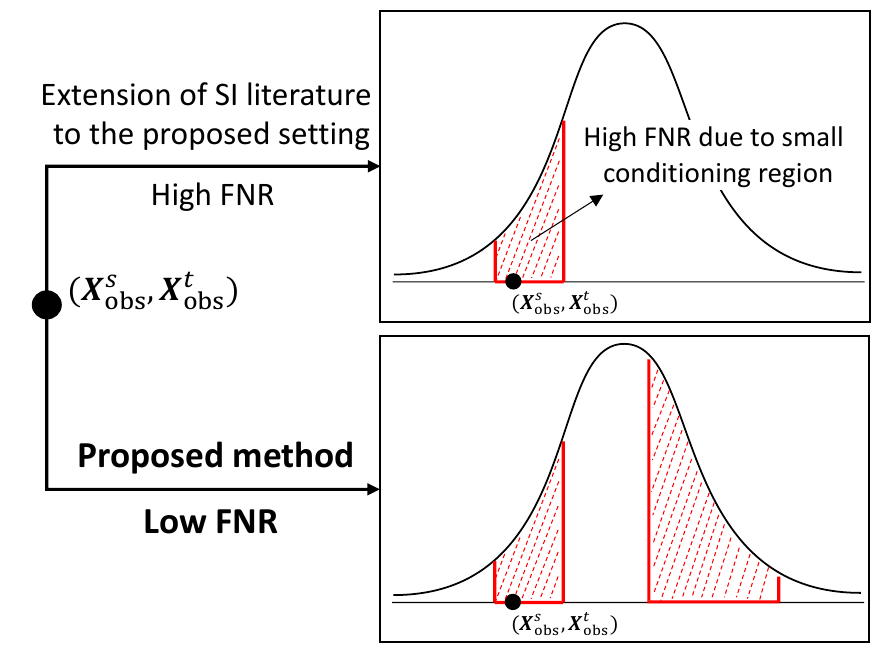} 
\caption{
\footnotesize 
If we leverage the idea of existing SI literature and apply to our setting, the FNR will be high due to small conditioning region by extra-conditioning. In contrast, the FNR is minimized with the proposed method.}
\label{fig:truncated_distributions}
\end{minipage}
\vspace{-8pt}
\end{figure*}

%
%
%

\subsection{Conditional Data Space Characterization} \label{subsec:conditional_data_space}

We define the set of ${\bm X^s \choose \bm X^t } \in \RR^{n_s + n_t}$ that satisfies the conditions in \eq{eq:conditioning_event} as:
\begin{align} \label{eq:conditional_data_space}
	\hspace{-2.5mm}\cD = \left \{ 
	{\bm X^s \choose \bm X^t } \Big | ~
	\cO_{\bm X^s, \bm X^t}
	=
	\cO_{\rm obs}, 
	\cQ_{\bm X^s, \bm X^t}
	=
	\cQ_{\rm obs}
	\right \}. 
\end{align}
In fact, the conditional data space $\cD$ is \emph{a line} in $\RR^{n_s + n_t}$ as stated in the following lemma. 
\begin{lemma} \label{lemma:data_line}
The set $\cD$ in \eq{eq:conditional_data_space} can be rewritten using a scalar parameter $z \in \RR$ as follows:
\begin{align} \label{eq:conditional_data_space_line}
	\cD = \left \{ {\bm X^s \choose \bm X^t } = \bm a + \bm b z \mid z \in \cZ \right \},
\end{align}
where vector $\bm a = \cQ_{\rm obs}$, $\bm b$ is defined in \eq{eq:q_and_b}, and
\begin{align} \label{eq:cZ}
	\cZ = \Big \{ 
	z \in \RR 
	\mid 
	\cO_{\bm a + \bm b z} = \cO_{\rm obs}
	\Big \}.
\end{align}
Here, with a slight abuse of notation, 
$
\cO_{\bm a + \bm b z} = \cO_{{\bm X^s \choose \bm X^t }}
$
is equivalent to $\cO_{\bm X^s, \bm X^t}$.
\end{lemma}
\begin{proof}
The proof is deferred to Appendix \ref{appx:proof_lemma_data_line}.
\end{proof}

\begin{remark}
Lemma \ref{lemma:data_line} indicates that we need NOT consider the $(n_s + n_t)$-dimensional data space.
Instead, we need only consider the \emph{one-dimensional projected} data space $\cZ$ in \eq{eq:cZ}.
The fact of restricting the data to a line has been already implicitly exploited in \cite{lee2016exact}, and explicitly discussed in Sec. 6 of \cite{liu2018more}.
\end{remark}

\paragraph{Reformulation of the selective $p$-value.}
Let us consider a random variable and its observation:
\begin{align*}
	Z = \bm \eta_j^\top {\bm X^s \choose \bm X^t } \in \RR 
	~~ \text{and} ~~ 
	Z_{\rm obs} = \bm \eta_j^\top {\bm X^s_{\rm obs} \choose \bm X^t_{\rm obs} } \in \RR.
\end{align*}
%
%
%
Then, the selective $p$-value in (\ref{eq:selective_p}) can be rewritten as 
\begin{align} \label{eq:selective_p_reformulated}
	p^{\rm sel}_j = \mathbb{P}_{\rm H_{0, j}} \Big ( |Z| \geq |Z_{\rm obs}| \mid  Z \in \cZ \Big ).
\end{align}
Because $Z \sim \NN(0, \bm \eta_j^\top \Sigma \bm \eta_j)$ under the null hypothesis, $Z \mid Z \in \cZ$ follows a \emph{truncated} normal distribution. 
Once the truncation region $\cZ$ is identified, computation of the selective $p$-value in (\ref{eq:selective_p_reformulated}) is straightforward.
Therefore, the remaining task is to identify $\cZ$.


\subsection{Identification of Truncation Region $\cZ$} \label{subsec:identification_cZ}

As discussed in \S \ref{subsec:conditional_data_space}, to calculate the selective $p$-value \eq{eq:selective_p_reformulated}, we must identify the truncation region $\cZ$ in (\ref{eq:cZ}).
%
%
However, there is no direct way to identify $\cZ$.
To overcome the difficulty, we utilize the concept the ``\emph{divide-and-conquer}'' and introduce an approach (illustrated in Fig. \ref{fig:schematic_illustration}) to efficiently identify $\cZ$, described as follows:

$\bullet$ \textbf{Step 1 (extra-conditioning)}: breaking down the problem into multiple sub-problems by additionally conditioning the transportation for DA and all the steps of the AD after DA. Each sub-problem can be directly solved with a finite number of operations.

$\bullet$ \textbf{Step 2 (hierarchical line search)}: hierarchically combining multiple extra-conditioning steps and check the condition $\cO_{\bm a + \bm b z} = \cO_{\rm obs}$ to obtain $\cZ$.

Specifically, let us denote by $U$ a number of all possible transportations for DA along the parametrized line.
%
%
We denote by $V^u$ a number of all possible sets of steps performed by the AD algorithm after the $\cT_u$ transportation, $u \in [U]$.
The entire one-dimensional space $\RR$ can be decomposed as: 
\begin{align*}
	\RR
	& = 
	\bigcup \limits_{u \in [U]}
	\bigcup \limits_{v \in [V^u]}
	\Bigg \{ 
	\underbrace{
	z \in \RR 
	\mid 
	\cT_{\bm a + \bm b z} = \cT_u,
	\cS_{\bm a + \bm b z} = \cS_v}_{\text{a sub-problem of extra-conditioning}}
	\Bigg \},
\end{align*}
where $\cT_{\bm a + \bm b z}$ denotes the OT-based DA on $\bm a + \bm b z$, $\cS_{\bm a + \bm b z}$ denote a set of steps performed by the AD algorithm after DA.
Our goal is to search a set 
\begin{align}\label{eq:cW}
	\cW = 
	\Big \{ 
		(u, v) : \cM(\cT_u, \cS_v) = \cO_{\rm obs},
	\Big \}, 
\end{align}
for all $u \in [U], v \in [V^u]$, the function $\cM$ is defined as:
\[
\cM : \big (\cT_{\bm a + \bm b z}, \cS_{\bm a + \bm b z} \big ) \mapsto \cO_{\bm a + \bm b z}.
\]
Finally, the region $\cZ$ in \eq{eq:cZ} can be obtained as follows:
\begin{align}
	\hspace{-2mm}\cZ 
	& = \Big \{ 
	z \in \RR 
	\mid 
	\cO_{\bm a + \bm b z} = \cO_{\rm obs}
	\Big \} \nonumber \\ 
	& = 
	\bigcup \limits_{(u, v) \in \cW}
	\Big \{ 
	z \in \RR 
	\mid 
	\cT_{\bm a + \bm b z} = \cT_u,
	\cS_{\bm a + \bm b z} = \cS_v
	\Big \}. \label{eq:cZ_new}
\end{align}
%
%

\begin{algorithm}[!t]
\renewcommand{\algorithmicrequire}{\textbf{Input:}}
\renewcommand{\algorithmicensure}{\textbf{Output:}}
\begin{footnotesize}
 \begin{algorithmic}[1]
  \REQUIRE $\bm a, \bm b, z_{\rm min}, z_{\rm max}$
	\vspace{4pt}
	\STATE Initialization: $u = 1$, $v = 1$, $z_{u, v}  = z_{\rm min}$, $\cW = \emptyset$
	\vspace{4pt}
	\WHILE {$z_{u, v} < z_{\rm max}$}
		\vspace{4pt}
		\STATE $\cT_u \leftarrow$ DA on data $\bm a + \bm b z_{u, v}$
		\vspace{4pt}
		\STATE Compute $[L_u, R_u] = \cZ_u \leftarrow$ Lemma \ref{lemma:cZ_1} 
		\vspace{4pt}
		\WHILE {\emph{true}}
		\vspace{4pt}
		\STATE $\cS_v \leftarrow$ AD after DA on data $\bm a + \bm b z_{u, v}$
		\vspace{4pt}
		\STATE Compute $\cZ_v \leftarrow$ Lemma \ref{lemma:cZ_2} 
		\vspace{4pt}
		\STATE $[L_{u, v}, R_{u, v}] = \cZ_{u, v} \leftarrow \cZ_u \cap \cZ_v$ 
		\vspace{4pt}
		\IF {$\cM(\cT_u, \cS_v) = \cO_{\rm obs}$}
		\vspace{4pt}
		\STATE $\cW \leftarrow \cW \cup \{ (u, v)\} $
		\vspace{4pt}
		\ENDIF
		\vspace{4pt}
		\IF {$R_{u, v} = R_u$}
		\vspace{4pt}
		\STATE $v \leftarrow 1$, $u \leftarrow u + 1$, $z_{u, v} = R_{u, v}$, \textbf{break}
		\vspace{4pt}
		\ENDIF
		%
		%
		\vspace{4pt}
		\STATE $v \leftarrow v + 1$, $z_{u, v} = R_{u, v}$
		\vspace{4pt}
		\ENDWHILE	
		\vspace{4pt}
	\ENDWHILE
	\vspace{4pt}
  \ENSURE $\cW$ 
 \end{algorithmic}
\end{footnotesize}
\caption{{\tt hierachical\_line\_search}}
\label{alg:hierachical_line_search}
\end{algorithm}

\textbf{Extra-conditioning (step 1).} For any $u \in [U]$ and $v \in [V^u]$, we define the subset of one-dimensional projected dataset on a line for the extra-conditioning as:
\begin{align} \label{eq:cZ_extra_condition}
	\cZ_{u, v} = 
	\big \{z \in \RR 
	\mid 
	\cT_{\bm a + \bm b z} = \cT_u,
	\cS_{\bm a + \bm b z} = \cS_v 
	\big \}.
\end{align}
The extra-conditioning region can be re-written as:
\begin{align*}
	&\cZ_{u, v}  = \cZ_u \cap \cZ_v, \\ 
	\cZ_u  = 
	\big \{ 
	z
	\mid 
	\cT_{\bm a + \bm b z} &= \cT_u
	\big \},  ~ ~
	\cZ_v = 
	\big \{ 
	z 
	\mid 
	\cS_{\bm a + \bm b z} = \cS_v
	\big \}.
\end{align*}
\begin{lemma} \label{lemma:cZ_1}
The set $\cZ_u$ can be characterized by a set of quadratic inequalities w.r.t. $z$ described as follows:
\begin{align*}
	\cZ_u
	= \Big \{ 
	z \in \RR 
	\mid 
	\bm w + \bm r z + \bm o z^2 \geq \bm 0
	\Big \},
\end{align*}
where vectors $\bm w$, $\bm r$, and $\bm o$ are defined in Appendix \ref{appx:proof_lemma_cZ_1}.
\end{lemma}


\begin{lemma} \label{lemma:cZ_2}
The set $\cZ_v$, which represents the set of operation performed by the MAD algorithm, can be characterized by a set of linear inequalities:
\begin{align} \label{eq:set_cZ_v}
	\cZ_v
	= \Big \{ 
	z \in \RR 
	\mid 
	\bm p z \leq \bm q
	\Big \},
\end{align}
where vectors $\bm p$ and $\bm q$ are provided in Appendix \ref{appx:proof_lemma_cZ_2}.

\end{lemma}


The proofs of Lemmas \ref{lemma:cZ_1} and \ref{lemma:cZ_2} are deferred to Appendices \ref{appx:proof_lemma_cZ_1} and \ref{appx:proof_lemma_cZ_2}.
Since the definitions of $\bm w$, $\bm r$, $\bm o$, $\bm p$, and $\bm q$ are complex and require extensive descriptions, we deferred them to the Appendices.
Briefly speaking, they are used for ensuring that the transportation and all the steps of the AD after DA remains the same for all $z \in \cZ_{u, v}$.
Lemmas \ref{lemma:cZ_1} and \ref{lemma:cZ_2} indicate that $\cZ_u$ and $\cZ_v$ can be \emph{analytically obtained} by solving the systems of quadratic and linear inequalities, respectively.
After computing $\cZ_u$ and $\cZ_v$, the extra conditioning region $\cZ_{u, v}$ in \eq{eq:cZ_extra_condition} is obtained by $\cZ_{u, v}  = \cZ_u \cap \cZ_v$.
As mentioned in \S \ref{sec:problem_setup}, our proposed CAD-AD can be applied to other AD algorithms whose operations can be characterized by sets of linear/quadratic inequalities (e.g., least absolute deviations, Huber regression).

\begin{remark}
The selective $p$-value computed with the extra-conditioning region $\cZ_{u, v}$ is still valid and this fact is well-known in the literature of conditional SI.
The computation of $\cZ_{u, v}$ can be considered as an extension of the methods presented in \cite{lee2016exact} and \cite{duy2021exact} into our proposed setting.
However, the major drawback of this case is that the TPR is low, i.e., the FNR is high.
Therefore, we introduce the line search step to remove the extra-conditioning for the purpose of minimizing the FNR.
The illustration is shown in Fig. \ref{fig:truncated_distributions}.
 
\end{remark}

\textbf{Hierarchical line search (step 2).}
Our strategy is to identify $\cW$ in \eq{eq:cW} by repeatedly applying OT-based DA and MAD after DA  to a sequence of datasets $\bm a + \bm b z$ within sufficiently wide range of $z \in [z_{\rm min}, z_{\rm max}]$\footnote{We set $z_{\rm min} = -20\sigma$ and $z_{\rm max} = 20 \sigma$, $\sigma$ is the standard deviation of the distribution of the test statistic, because the probability mass outside this range is negligibly small.}.
For simplicity, we consider the case in which $\cZ_u$ is an interval \footnote{If $\cZ_u$ is a union of intervals, we can select the interval containing the data point that we are currently considering.}. 
Since $\cZ_v$ is also an interval, $\cZ_{u, v}$ is an interval.
We denote $\cZ_u = [L_u, R_u]$ and $\cZ_{u, v} = [L_{u, v}, R_{u, v}]$.
The hierarchical line search procedure can be summarized in Algorithm \ref{alg:hierachical_line_search}.
After obtaining $\cW$ by Algorithm \ref{alg:hierachical_line_search}.
We can compute $\cZ$ in \eq{eq:cZ_new}, which is subsequently used to obtain the proposed selective $p$-value in \eq{eq:selective_p_reformulated}.
The entire steps of the proposed CAD-DA method is summarized in Algorithm \ref{alg:cad_da}.

\begin{algorithm}[!t]
\renewcommand{\algorithmicrequire}{\textbf{Input:}}
\renewcommand{\algorithmicensure}{\textbf{Output:}}
\begin{footnotesize}
 \begin{algorithmic}[1]
  \REQUIRE $\bm X^s_{\rm obs}, \bm X^t_{\rm obs}, z_{\rm min}, z_{\rm max}$
	\vspace{4pt}
	\STATE $\cO_{\rm obs} \leftarrow$ AD after DA on $\big \{ \bm X^s_{\rm obs}, \bm X^t_{\rm obs} \big \} $
	\vspace{4pt}
	\FOR {$j \in \cO_{\rm obs}$}
		\vspace{4pt}
		\STATE Compute $\bm \eta_j \leftarrow$ Eq. \eq{eq:eta_j}, $\bm a$ and $\bm b \leftarrow$ Eq. \eq{eq:conditional_data_space_line}
		\vspace{4pt}
		\STATE $\cW \leftarrow$ {\tt hierachical\_line\_search} ($\bm a, \bm b, z_{\rm min}, z_{\rm max}$)
		\vspace{4pt}
		\STATE $p^{\rm sel}_j \leftarrow$  Eq. \eq{eq:selective_p_reformulated} with $\cZ \leftarrow$ Eq. \eq{eq:cZ_new}
		\vspace{4pt}
	\ENDFOR
	\vspace{4pt}
  \ENSURE $\big \{ p^{\rm sel}_j \big \}_{j \in \cO_{\rm obs}}$ 
 \end{algorithmic}
\end{footnotesize}
\caption{{\tt CAD-DA}}
\label{alg:cad_da}
\end{algorithm}

\subsection{Extension to Multi-Dimension} \label{subsec:extension}
In this section, we generalize the problem setup and the proposed method in multi-dimension. 
We consider two random sets $X^s \in \RR^{n_s \times d}$ and $X^t \in  \RR^{n_t \times d}$ of $d$-dimensional vectors 
%
%
which are random samples from
%
\begin{align*}
	{\rm {vec}} (X^s) &\sim \NN \Big ({\rm vec} (M^s), ~ \Sigma^s_{{\rm vec} (X^s)} \Big ), \\
	{\rm {vec}} (X^t) &\sim \NN \Big ({\rm vec} (M^t), ~ \Sigma^t_{{\rm vec} (X^t)} \Big ),
\end{align*}
%
where ${\rm vec}(\cdot)$ is an operator that transforms a matrix into a vector with concatenated rows, $M^s$ and $M^t$ are signal matrices that are unknown, $\Sigma^s_{{\rm vec} (X^s)}$ and $\Sigma^t_{{\rm vec} (X^t)}$ are variance matrices that are known. 
%
%
After obtaining $\hat{T}$ by solving the OT problem, 
we obtain 
$\tilde{X}^s = n_s \hat{T} X^t$.
Then, the AD result after DA is:
\begin{align*}
	\cA: \big \{ \tilde{X}^s, X^t \big \} 
		\mapsto
		\cO \in [n_t].
\end{align*}
For $j \in \cO$, we consider the following hypotheses:
\begin{align*}
	{\rm H}_{0, j}:  M^t_{j, \kappa} = \bar{ M}^t_{- \cO, \kappa}
	\quad
	\text{vs.}
	\quad 
	{\rm H}_{1, j}:  M^t_{j, \kappa} \neq \bar{M}^t_{- \cO, \kappa},
\end{align*}
for all $\kappa \in [d]$, 
%
$
	\bar{M}^t_{- \cO, \kappa} = 
	\frac{1}{n_t - |\cO|} \sum \limits_{\ell \in [n_t] \setminus \cO}
	M^t_{\ell, \kappa}.
$

The test-statistic is then can be defined as:
\begin{align} \label{eq:test_statistic_mul}
	\Gamma_j = \sum_{\kappa \in [d]}
	|X^t_{j, \kappa} - \bar{X}^t_{- \cO, \kappa}| 
\end{align}
In order to compute the valid $p$-values, 
we consider the following conditional distribution of the test statistic:
\begin{align} \label{eq:conditional_distribution_mul}
	\mathbb{P} \Big ( 
	\Gamma_j
	\mid
	\cO_{X^s, X^t}
	=
	\cO_{\rm obs}, ~
	\cI_{X^s, X^t} 
	=
	\cI_{\rm obs}
	\Big ),
\end{align}
where $\cI_{X^s, X^t}$ is a set of all the signs of the subtractions in \eq{eq:test_statistic_mul}.
By extending the techniques in \S \ref{subsec:conditional_data_space} and \S \ref{subsec:identification_cZ}, 
the conditional space in \eq{eq:conditional_distribution_mul} can be obtained which is subsequently used to compute the $p$-value.


\section{Experiment} \label{sec:experment}

In this section, we demonstrate the performance of the proposed method.
We compared the performance of the following methods in terms of FPR and TPR:

$\bullet$ {\tt CAD-DA}: proposed method

$\bullet$ {\tt CAD-DA-oc}: proposed method with only the extra-conditioning described in \S \ref{subsec:conditional_data_space} (extension the ideas in \cite{lee2016exact, duy2021exact} to our proposed setting)

$\bullet$ {\tt Bonferroni}: the most popular multiple hypothesis testing approach

$\bullet$ {\tt Naive}: traditional statistical inference

$\bullet$ {\tt No Inference}: AD after DA without inference

We note that if a method cannot control the FPR under $\alpha$, it is \emph{invalid} and we do not need to consider its TPR.
A method has high TPR indicates that it has low FNR.
We set $\alpha = 0.05$.
We executed the code on Intel(R) Xeon(R) CPU E5-2687W v4 @ 3.00GHz.

\subsection{Numerical Experiments.}

\textbf{Univariate case.} We generated $\bm X^s$ and $\bm X^t$ with $\mu^s_i = 0$, $\veps^s_i \sim \NN(0, 1)$, for all $i \in [n_s]$, and $ \mu^t_j = 2, \veps^t_j \sim \NN(0, 1)$, for all $j \in [n_t]$.
We randomly selected $5$ data points in the target domain and made them to be abnormal by setting $\mu^t_j = \mu^t_j + \Delta$.
Regarding the FPR experiments, we set $n_s \in \{ 50, 100, 150, 200 \}$, $n_t = 25$, and $\Delta = 0$.
In regard to the TPR experiments, we set $n_s = 150$, $n_t = 25$, and $\Delta \in \{ 1, 2, 3, 4\} $. 
Each experiment was repeated 120 times.
The results are shown in Fig. \ref{fig:fpr_tpr_1d}.
In the plot on the left, the {\tt CAD-DA}, {\tt CAD-DA-oc} and {\tt Bonferroni} controlled the FPR under $\alpha$ whereas the {\tt Naive} and {\tt No Inference} \emph{could not}. 
Because the {\tt Naive} and {\tt No Inference} failed to control the FPR, we no longer considered the TPR .
In the plot on the left, we can see that the {\tt CAD-DA} has highest TPR compared to other methods in all the cases, i.e., the {\tt CAD-DA} has lowest FNR compared to the competitors.

\begin{figure}[!t]
     \centering
     \begin{subfigure}[b]{0.492\linewidth}
         \centering
         \includegraphics[width=\textwidth]{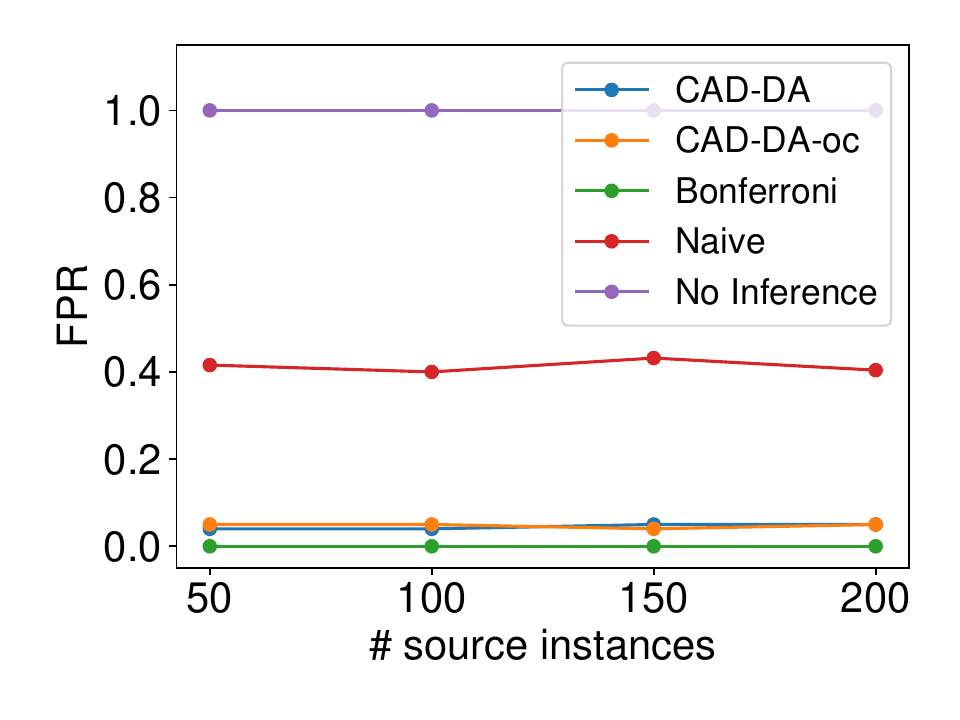}
         \caption{FPR}
     \end{subfigure}
     \hfill
     \begin{subfigure}[b]{0.492\linewidth}
         \centering
         \includegraphics[width=\textwidth]{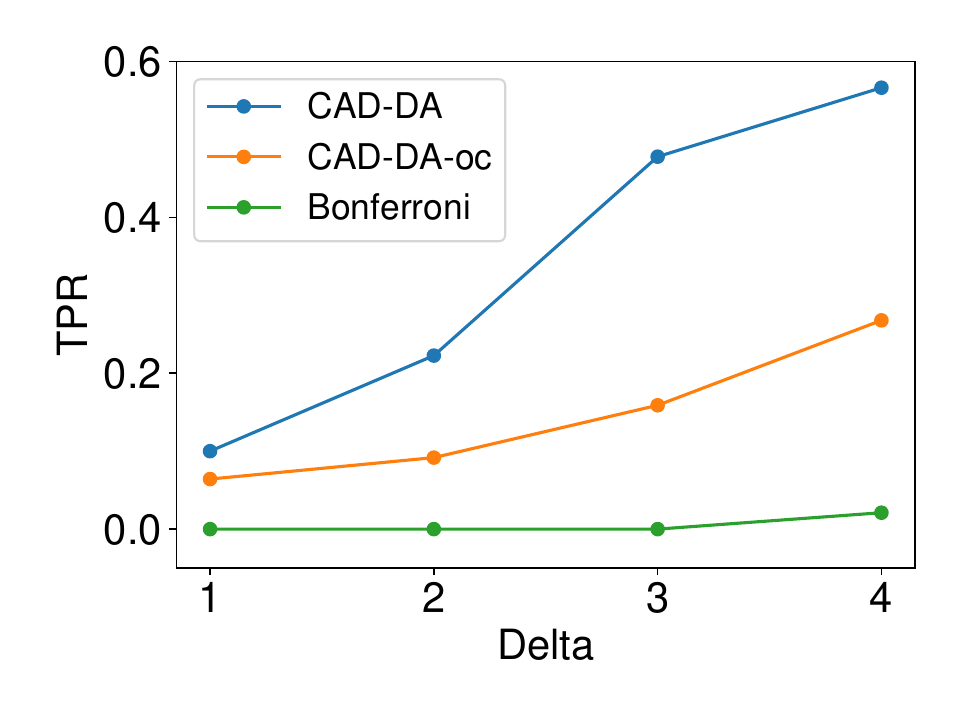}
         \caption{TPR}
     \end{subfigure}
     \caption{FPR and TPR in univariate case}
     \label{fig:fpr_tpr_1d}
     
\end{figure}

\begin{figure}[!t]
     \centering
     \begin{subfigure}[b]{0.492\linewidth}
         \centering
         \includegraphics[width=\textwidth]{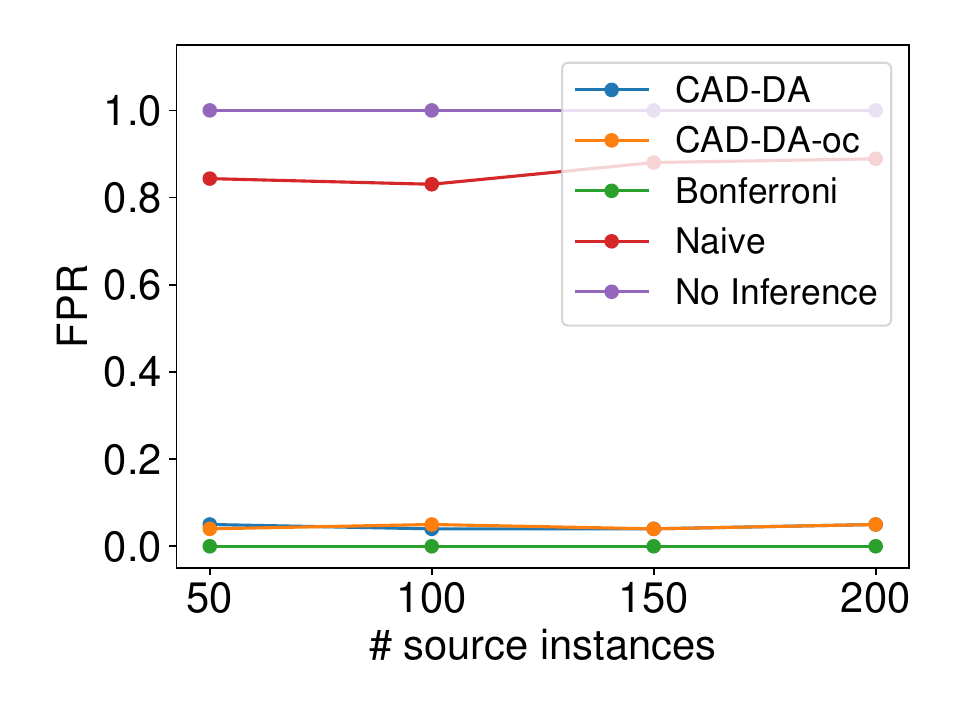}
         \caption{FPR}
     \end{subfigure}
     \hfill
     \begin{subfigure}[b]{0.492\linewidth}
         \centering
         \includegraphics[width=\textwidth]{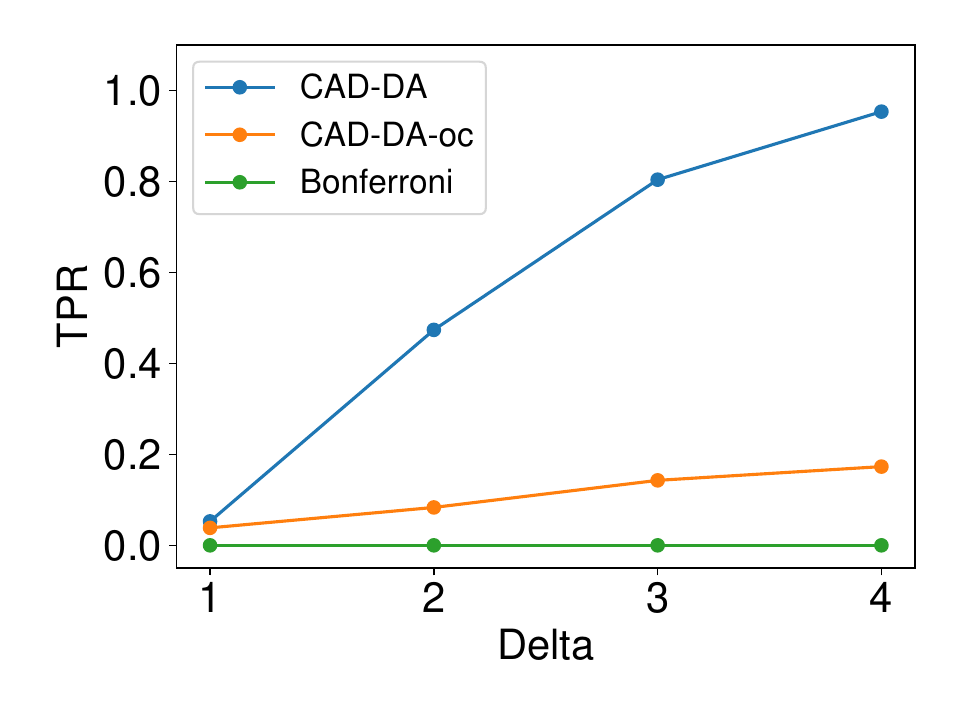}
         \caption{TPR}
     \end{subfigure}
     \caption{FPR and TPR in multi-dimensional case}
     \label{fig:fpr_tpr_10d}
     \vspace{-8pt}
\end{figure}

\textbf{Multi-dimensional case.} We generated $X^s$ and $X^t$ with $X^s_{i, :} \sim \NN(\bm 0_d, I_d), \forall i \in [n_s]$, $X^t_{j, :} \sim \NN(\bm 2_d, I_d), \forall j \in [n_t]$, and the dimension $d = 10$.
The settings 
were the same as univariate case.
The results are shown in Fig. \ref{fig:fpr_tpr_10d}.
%
%
The important point we would like to note is that the difference in TPR between the {\tt CAD-DA} and {\tt CAD-DA-oc} in the multi-dimensional case is larger than that in the univariate case.
This indicates that the issue of extra-conditioning is more serious when increasing $d$.
If we naively extend the ideas in \cite{lee2016exact, duy2021exact} to our setting, the TPR will be low, i.e., the FNR is high.
In contrast, by introducing the hierarchical linear search step to remove the extra-conditioning, the FNR is significantly decreased by the proposed {\tt CAD-DA} method.

\textbf{Correlated data.} In this setting, we consider the case where the data is not independent.
We generated $X^s$ and $X^t$ with $X^s_{i, :} \sim \NN(\bm 0_d, \Xi), \forall i \in [n_s]$, $X^t_{j, :} \sim \NN(\bm 2_d, \Xi), \forall j \in [n_t]$, the matrix 
$
\Xi = \left [\rho^{|i - j|} \right ]_{ij}, \forall i,j \in [d],
$
$\rho = 0.5$, and $d = 10$.
The settings for FPR and TPR experiments were also the same as univariate case.
The results are shown in Fig. \ref{fig:fpr_tpr_correlated}.
Additionally, we also conducted FPR and TPR experiments when changing $\rho \in \{ 0.2, 0.4, 0.6, 0.8\}$. We set $n_s = 150, n_t = 25$, $\Delta = 0$ for FPR experiments, and $\Delta = 4$ for the TPR experiments.
The results are shown in Fig. \ref{fig:fpr_tpr_correlated_change_rho}.
In essence,  the correlated data contains redundant information. 
This means that the effective sample size is smaller than the actual sample size. 
A smaller effective sample size reduces the amount of information available for the statistical test, making it less powerful. 
Therefore, the TPR tends to decrease when increasing the value of $\rho$.
However, in all the case, the proposed {\tt CAD-DA} method consistently achieves the highest TPR while controlling the FPR.

\begin{figure}[!t]
     \centering
     \begin{subfigure}[b]{0.492\linewidth}
         \centering
         \includegraphics[width=\textwidth]{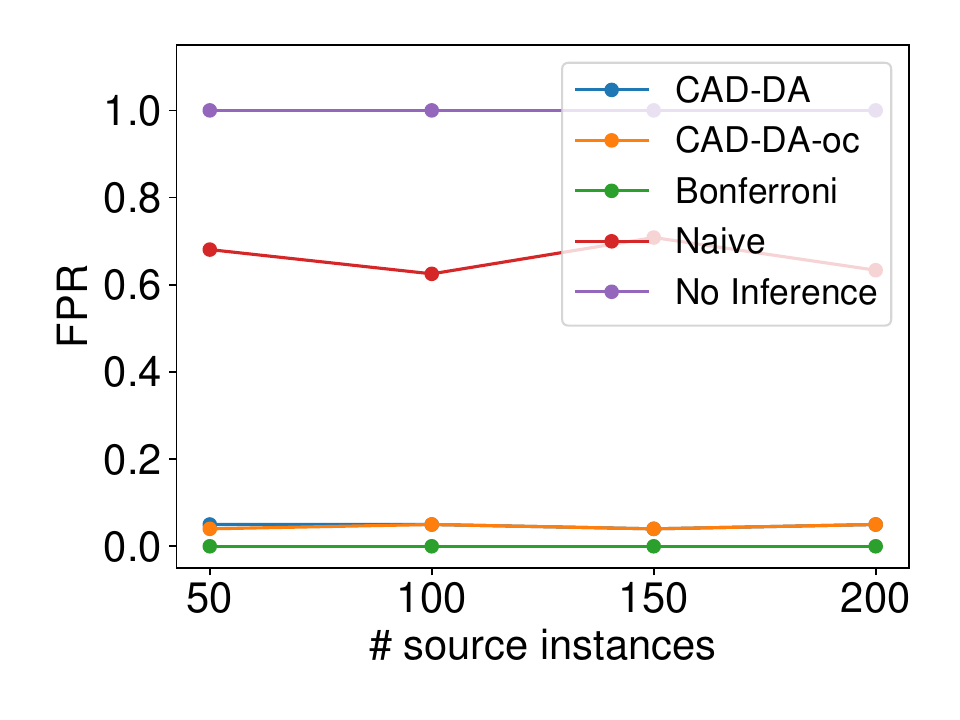}
         \caption{FPR}
     \end{subfigure}
     \hfill
     \begin{subfigure}[b]{0.492\linewidth}
         \centering
         \includegraphics[width=\textwidth]{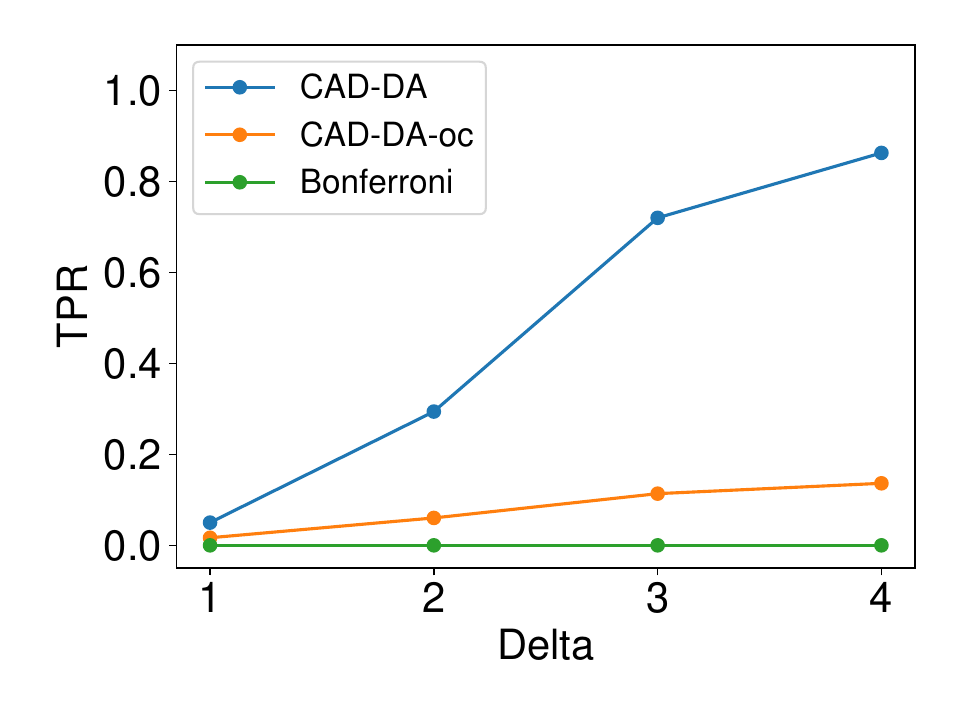}
         \caption{FPR}
     \end{subfigure}
     \caption{FPR and TPR in the case of correlated data}
     \label{fig:fpr_tpr_correlated}
\end{figure}

\begin{figure}[!t]
     \centering
     \begin{subfigure}[b]{0.492\linewidth}
         \centering
         \includegraphics[width=\textwidth]{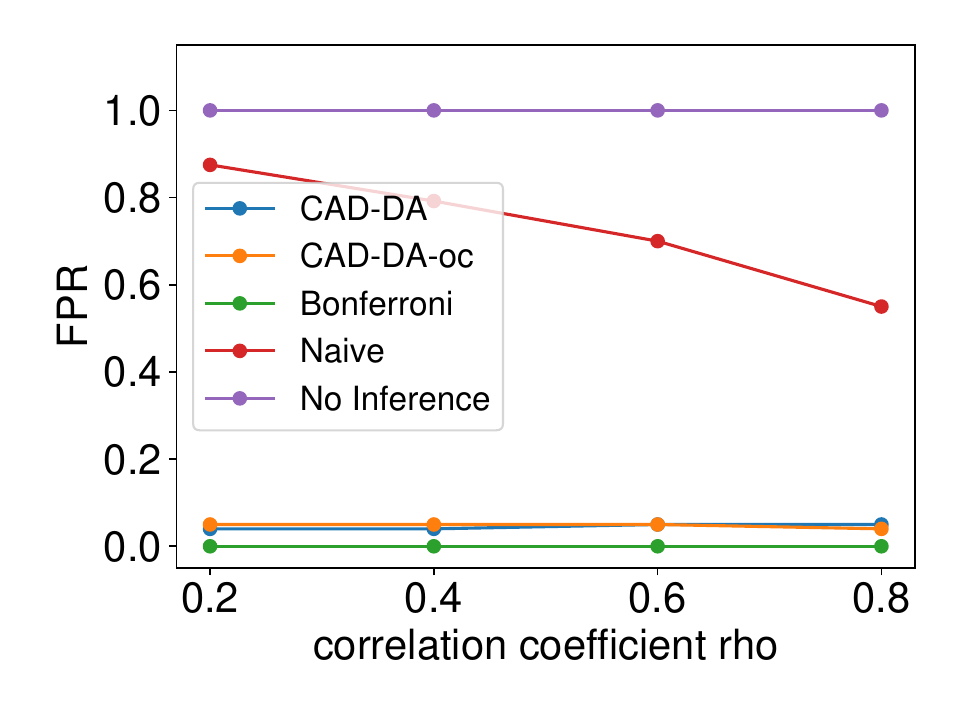}
         \caption{FPR}
     \end{subfigure}
     \hfill
     \begin{subfigure}[b]{0.492\linewidth}
         \centering
         \includegraphics[width=\textwidth]{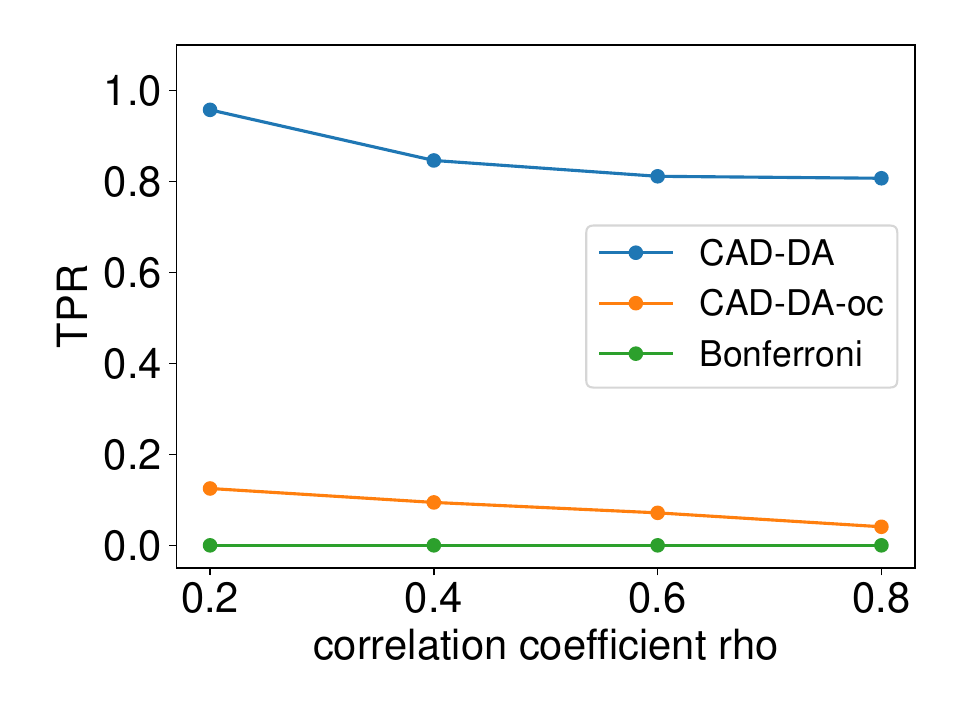}
         \caption{TPR}
     \end{subfigure}
     \caption{FPR and TPR when changing  $\rho$}
     \label{fig:fpr_tpr_correlated_change_rho}
\end{figure}

\textbf{Comparison with \cite{tsukurimichi2022conditional} and robustness experiments.}
Since the method of \cite{tsukurimichi2022conditional} is not applicable to our proposed setting, we have to introduce an extended setting for conducting the experiments of FPR comparison with their method.
The details are provided in Appendix \ref{appx:comparison_with_tsukurimichi}.
The results are shown in Fig. \ref{fig:comparison_with_tsukurimichi}.
The proposed {\tt CAD-DA} could properly control the FPR under $\alpha = 0.05$ whereas the existing method \cite{tsukurimichi2022conditional} failed because it could not account for the influence of DA process.
Additionally, we conducted the experiments on the robustness of the {\tt CAD-DA} when the noise follows Laplace distribution, skew normal distribution, $t_{\rm 20}$ distribution, and variance is estimated from the data. 
The details are shown in Appendix \ref{appx:robustness}.
Overall, the {\tt CAD-DA} method still maintained good performance.

\vspace{-4pt}

\subsection{Real-data Experiments}

\vspace{-2pt}

We performed power comparison on real-data.
We used the \emph{NeuroKit2} simulator \cite{Makowski2021neurokit} to generate realistic respiration signals ($n_s = 150$) used for source dataset, and heart-beat signals ($n_t = 25$) used for target dataset, each with a dimensionality of $d = 12$.
We repeated the experiment $N \in \{ 120, 240\} $ times.
The results are shown in Tab. \ref{tbl:real_data_1}.
While the {\tt Bonferroni}, {\tt CAD-DA-oc} and {\tt CAD-DA} could control the FPR, the {\tt CAD-DA} had the highest TPR in all the cases.
We note that, in the case of Bonferroni correction, the TPRs were all 0.0 which indicates that this method is too conservative and all the true anomalies were missed out even though there exists, i.e., FNR = 1.0.

\begin{figure}[!t]
\begin{minipage}{0.48\linewidth}
  \centering
  \includegraphics[width=\linewidth]{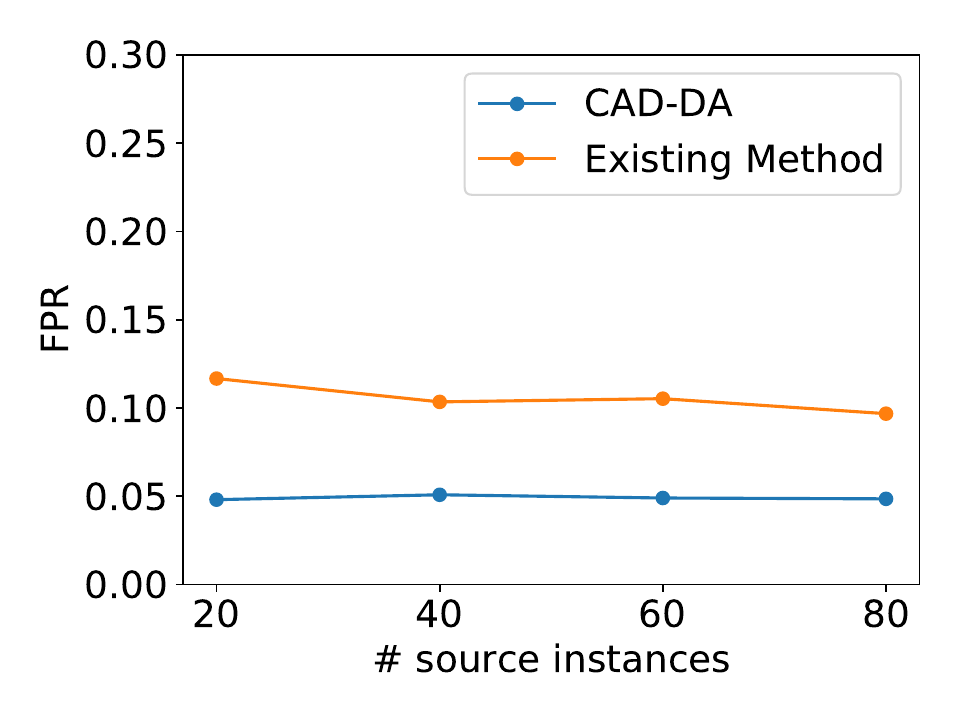}  
\caption{FPR comparison with \cite{tsukurimichi2022conditional}}
\label{fig:comparison_with_tsukurimichi}
\end{minipage}
\hspace{0.5mm}
\begin{minipage}{0.48\linewidth}
  \centering
  \includegraphics[width=\linewidth]{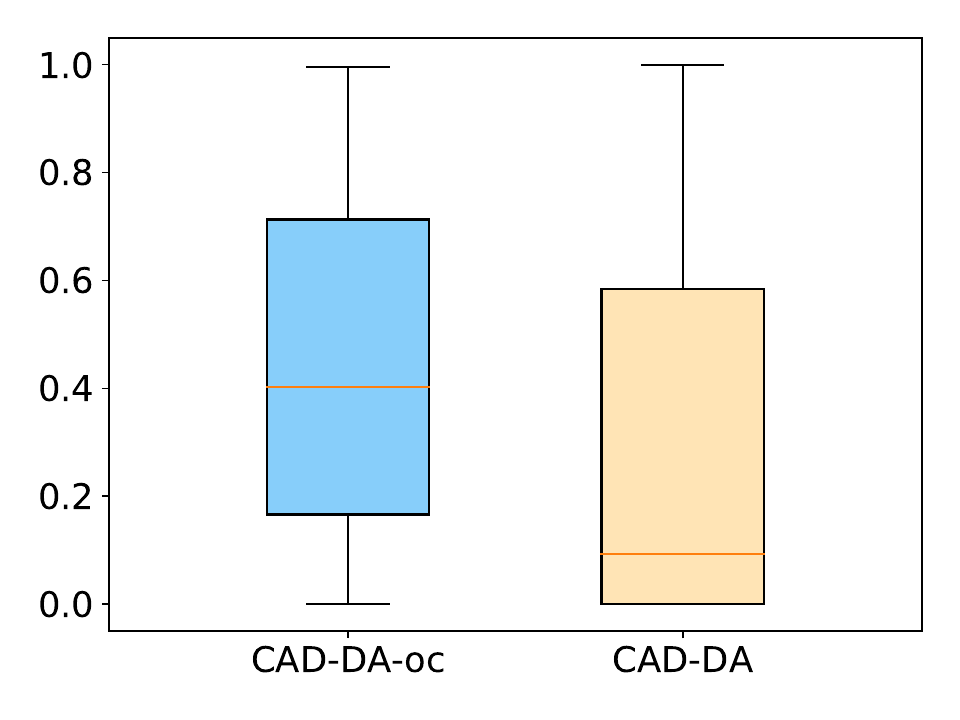}  
\caption{Boxplots of $p$-values}
\label{fig:boxplot}
\end{minipage}
\end{figure}

\begin{table} [!t]
\renewcommand{\arraystretch}{1.1}
\centering
\caption{Comparison on real-data.}
\begin{tabular}{ |l|c|c|c|c| } 
  \hline
  & \multicolumn{2}{|c|}{$N = 120$} & \multicolumn{2}{|c|}{$N = 240$} \\ 
  \hline
  & ~ FPR ~&~TPR~&~FPR~&~TPR~\\
  \hline
   \hline
 \textbf{No Inference} & 1.00 & N/A & 1.00 & N/A \\ 
  \hline
 \textbf{Naive} & 0.90 & N/A & 0.90 & N/A \\ 
 \hline
 \textbf{Bonferroni} & \textbf{0.00} & 0.00 & \textbf{0.00} &  0.00 \\ 
  \hline
 \textbf{CAD-DA-oc} & \textbf{0.05} & 0.12 & \textbf{0.04} & 0.07 \\ 
  \hline
 \textbf{CAD-DA} & \textbf{0.04} & \textbf{0.71} & \textbf{0.04} & \textbf{0.73}\\ 
 \hline
\end{tabular}
\label{tbl:real_data_1}
\vspace{-10pt}
\end{table}

Additionally, we compared the $p$-values of the {\tt CAD-DA-oc} and {\tt CAD-DA} on Heart Disease dataset, which is available at the UCI Machine Learning Repository.
We randomly selected $n_s = 150$ patients whose gender is male (source domain), $n_t = 25$ patients whose gender is female (target), and conducted the experiment to compute the $p$-value.
The experiments was repeated 120 times. The boxplots of the distribution of the $p$-values are illustrated in Fig. \ref{fig:boxplot}.
The $p$-values of the {\tt CAD-DA} tend to be smaller than those of {\tt CAD-DA-oc}, which indicates that the proposed {\tt CAD-DA} method has higher power than the {\tt CAD-DA-oc}.


\vspace{-5pt}

\section{Discussion} \label{sec:discussion}

\vspace{-5pt}
We propose a novel setup of testing the results of AD after DA and introduce a method to compute a valid $p$-value for conducting the statistical test.
We believe that this study stands as a significant step toward controllable machine leaning under DA. 
Some open questions remain. 
Our method currently does not support MMD-based DAs \cite{gong2013connecting, pan2010domain, baktashmotlagh2013unsupervised} or metric learning-based DA \cite{saenko2010adapting} because the selection event of these methods are more complicated than OT-based DA.
A potential solution could involve a sampling-based approach to approximate the truncation region for $p$-value computation.
Similarly, our method is applicable to AD algorithms with analytically characterizable event, such as MAD and  robust regression methods. 
Extending the proposed method to more complex AD algorithms would represent a valuable contribution.

\bibliographystyle{abbrv}
\bibliography{ref}

\begin{thebibliography}{10}

\bibitem{aggarwal2005abnormality}
C.~C. Aggarwal.
\newblock On abnormality detection in spuriously populated data streams.
\newblock In {\em Proceedings of the 2005 siam international conference on data
  mining}, pages 80--91. SIAM, 2005.

\bibitem{aggarwal2017outlier}
C.~C. Aggarwal.
\newblock {\em Outlier Analysis}.
\newblock Springer, 2017.

\bibitem{andrews1974robust}
D.~F. Andrews.
\newblock A robust method for multiple linear regression.
\newblock {\em Technometrics}, 16(4):523--531, 1974.

\bibitem{avci2021review}
O.~Avci, O.~Abdeljaber, S.~Kiranyaz, M.~Hussein, M.~Gabbouj, and D.~J. Inman.
\newblock A review of vibration-based damage detection in civil structures:
  From traditional methods to machine learning and deep learning applications.
\newblock {\em Mechanical systems and signal processing}, 147:107077, 2021.

\bibitem{baktashmotlagh2013unsupervised}
M.~Baktashmotlagh, M.~T. Harandi, B.~C. Lovell, and M.~Salzmann.
\newblock Unsupervised domain adaptation by domain invariant projection.
\newblock In {\em Proceedings of the IEEE international conference on computer
  vision}, pages 769--776, 2013.

\bibitem{chen2019valid}
S.~Chen and J.~Bien.
\newblock Valid inference corrected for outlier removal.
\newblock {\em Journal of Computational and Graphical Statistics}, pages 1--12,
  2019.

\bibitem{chen2022selective}
Y.~T. Chen and D.~M. Witten.
\newblock Selective inference for k-means clustering.
\newblock {\em arXiv preprint arXiv:2203.15267}, 2022.

\bibitem{das2021fast}
D.~Das, V.~N.~L. Duy, H.~Hanada, K.~Tsuda, and I.~Takeuchi.
\newblock Fast and more powerful selective inference for sparse high-order
  interaction model.
\newblock {\em arXiv preprint arXiv:2106.04929}, 2021.

\bibitem{du2020damage}
Y.~Du, S.~Zhou, X.~Jing, Y.~Peng, H.~Wu, and N.~Kwok.
\newblock Damage detection techniques for wind turbine blades: A review.
\newblock {\em Mechanical Systems and Signal Processing}, 141:106445, 2020.

\bibitem{duy2022quantifying}
V.~N.~L. Duy, S.~Iwazaki, and I.~Takeuchi.
\newblock Quantifying statistical significance of neural network-based image
  segmentation by selective inference.
\newblock {\em Advances in Neural Information Processing Systems},
  35:31627--31639, 2022.

\bibitem{duy2022exact}
V.~N.~L. Duy and I.~Takeuchi.
\newblock Exact statistical inference for time series similarity using dynamic
  time warping by selective inference.
\newblock {\em arXiv preprint arXiv:2202.06593}, 2022.

\bibitem{duy2021more}
V.~N.~L. Duy and I.~Takeuchi.
\newblock More powerful conditional selective inference for generalized lasso
  by parametric programming.
\newblock {\em The Journal of Machine Learning Research}, 23(1):13544--13580,
  2022.

\bibitem{duy2021exact}
V.~N.~L. Duy and I.~Takeuchi.
\newblock Exact statistical inference for the wasserstein distance by selective
  inference.
\newblock {\em Annals of the Institute of Statistical Mathematics},
  75(1):127--157, 2023.

\bibitem{duy2020computing}
V.~N.~L. Duy, H.~Toda, R.~Sugiyama, and I.~Takeuchi.
\newblock Computing valid p-value for optimal changepoint by selective
  inference using dynamic programming.
\newblock In {\em Advances in Neural Information Processing Systems}, 2020.

\bibitem{fithian2014optimal}
W.~Fithian, D.~Sun, and J.~Taylor.
\newblock Optimal inference after model selection.
\newblock {\em arXiv preprint arXiv:1410.2597}, 2014.

\bibitem{fithian2015selective}
W.~Fithian, J.~Taylor, R.~Tibshirani, and R.~Tibshirani.
\newblock Selective sequential model selection.
\newblock {\em arXiv preprint arXiv:1512.02565}, 2015.

\bibitem{flamary2016optimal}
R.~Flamary, N.~Courty, D.~Tuia, and A.~Rakotomamonjy.
\newblock Optimal transport for domain adaptation.
\newblock {\em IEEE Trans. Pattern Anal. Mach. Intell}, 1:1--40, 2016.

\bibitem{gao2022selective}
L.~L. Gao, J.~Bien, and D.~Witten.
\newblock Selective inference for hierarchical clustering.
\newblock {\em Journal of the American Statistical Association}, pages 1--11,
  2022.

\bibitem{gong2013connecting}
B.~Gong, K.~Grauman, and F.~Sha.
\newblock Connecting the dots with landmarks: Discriminatively learning
  domain-invariant features for unsupervised domain adaptation.
\newblock In {\em International conference on machine learning}, pages
  222--230. PMLR, 2013.

\bibitem{huber1973robust}
P.~J. Huber et~al.
\newblock Robust regression: asymptotics, conjectures and monte carlo.
\newblock {\em The annals of statistics}, 1(5):799--821, 1973.

\bibitem{hyun2018exact}
S.~Hyun, M.~G’sell, and R.~J. Tibshirani.
\newblock Exact post-selection inference for the generalized lasso path.
\newblock {\em Electronic Journal of Statistics}, 12(1):1053--1097, 2018.

\bibitem{hyun2018post}
S.~Hyun, K.~Lin, M.~G'Sell, and R.~J. Tibshirani.
\newblock Post-selection inference for changepoint detection algorithms with
  application to copy number variation data.
\newblock {\em arXiv preprint arXiv:1812.03644}, 2018.

\bibitem{inoue2017post}
S.~Inoue, Y.~Umezu, S.~Tsubota, and I.~Takeuchi.
\newblock Post clustering inference for heterogeneous data.
\newblock {\em IEICE Technical Report; IEICE Tech. Rep.}, 117(293):69--76,
  2017.

\bibitem{jewell2022testing}
S.~Jewell, P.~Fearnhead, and D.~Witten.
\newblock Testing for a change in mean after changepoint detection.
\newblock {\em Journal of the Royal Statistical Society Series B: Statistical
  Methodology}, 84(4):1082--1104, 2022.

\bibitem{lee2016exact}
J.~D. Lee, D.~L. Sun, Y.~Sun, and J.~E. Taylor.
\newblock Exact post-selection inference, with application to the lasso.
\newblock {\em The Annals of Statistics}, 44(3):907--927, 2016.

\bibitem{lee2015evaluating}
J.~D. Lee, Y.~Sun, and J.~E. Taylor.
\newblock Evaluating the statistical significance of biclusters.
\newblock {\em Advances in neural information processing systems}, 28, 2015.

\bibitem{liu2018more}
K.~Liu, J.~Markovic, and R.~Tibshirani.
\newblock More powerful post-selection inference, with application to the
  lasso.
\newblock {\em arXiv preprint arXiv:1801.09037}, 2018.

\bibitem{loftus2014significance}
J.~R. Loftus and J.~E. Taylor.
\newblock A significance test for forward stepwise model selection.
\newblock {\em arXiv preprint arXiv:1405.3920}, 2014.

\bibitem{Makowski2021neurokit}
D.~Makowski, T.~Pham, Z.~J. Lau, J.~C. Brammer, F.~Lespinasse, H.~Pham,
  C.~Sch{\"o}lzel, and S.~H.~A. Chen.
\newblock Neurokit2: A python toolbox for neurophysiological signal processing.
\newblock {\em Behavior Research Methods}, Feb 2021.

\bibitem{maronna2019robust}
R.~A. Maronna, R.~D. Martin, V.~J. Yohai, et~al.
\newblock {\em Robust statistics: theory and methods (with R)}.
\newblock John Wiley \& Sons, 2019.

\bibitem{miwa2023valid}
D.~Miwa, D.~V.~N. Le, and I.~Takeuchi.
\newblock Valid p-value for deep learning-driven salient region.
\newblock In {\em Proceedings of the 11th International Conference on Learning
  Representation}, 2023.

\bibitem{neufeld2022tree}
A.~C. Neufeld, L.~L. Gao, and D.~M. Witten.
\newblock Tree-values: selective inference for regression trees.
\newblock {\em Journal of Machine Learning Research}, 23(305):1--43, 2022.

\bibitem{pan1995multiple}
J.-X. Pan and K.-T. Fang.
\newblock Multiple outlier detection in growth curve model with unstructured
  covariance matrix.
\newblock {\em Annals of the Institute of Statistical Mathematics},
  47(1):137--153, 1995.

\bibitem{pan2010domain}
S.~J. Pan, I.~W. Tsang, J.~T. Kwok, and Q.~Yang.
\newblock Domain adaptation via transfer component analysis.
\newblock {\em IEEE transactions on neural networks}, 22(2):199--210, 2010.

\bibitem{pourhabibi2020fraud}
T.~Pourhabibi, K.-L. Ong, B.~H. Kam, and Y.~L. Boo.
\newblock Fraud detection: A systematic literature review of graph-based
  anomaly detection approaches.
\newblock {\em Decision Support Systems}, 133:113303, 2020.

\bibitem{rousseeuw2005robust}
P.~J. Rousseeuw and A.~M. Leroy.
\newblock {\em Robust regression and outlier detection}, volume 589.
\newblock John wiley \& sons, 2005.

\bibitem{rugamer2020inference}
D.~R{\"u}gamer and S.~Greven.
\newblock Inference for l 2-boosting.
\newblock {\em Statistics and computing}, 30(2):279--289, 2020.

\bibitem{saenko2010adapting}
K.~Saenko, B.~Kulis, M.~Fritz, and T.~Darrell.
\newblock Adapting visual category models to new domains.
\newblock In {\em Computer Vision--ECCV 2010: 11th European Conference on
  Computer Vision, Heraklion, Crete, Greece, September 5-11, 2010, Proceedings,
  Part IV 11}, pages 213--226. Springer, 2010.

\bibitem{srivastava1998outliers}
M.~S. Srivastava and D.~von Rosen.
\newblock Outliers in multivariate regression models.
\newblock {\em Journal of Multivariate Analysis}, 65(2):195--208, 1998.

\bibitem{sugiyama2021more}
K.~Sugiyama, V.~N. Le~Duy, and I.~Takeuchi.
\newblock More powerful and general selective inference for stepwise feature
  selection using homotopy method.
\newblock In {\em International Conference on Machine Learning}, pages
  9891--9901. PMLR, 2021.

\bibitem{sugiyama2021valid}
R.~Sugiyama, H.~Toda, V.~N.~L. Duy, Y.~Inatsu, and I.~Takeuchi.
\newblock Valid and exact statistical inference for multi-dimensional multiple
  change-points by selective inference.
\newblock {\em arXiv preprint arXiv:2110.08989}, 2021.

\bibitem{suzumura2017selective}
S.~Suzumura, K.~Nakagawa, Y.~Umezu, K.~Tsuda, and I.~Takeuchi.
\newblock Selective inference for sparse high-order interaction models.
\newblock In {\em Proceedings of the 34th International Conference on Machine
  Learning-Volume 70}, pages 3338--3347. JMLR. org, 2017.

\bibitem{tanizaki2020computing}
K.~Tanizaki, N.~Hashimoto, Y.~Inatsu, H.~Hontani, and I.~Takeuchi.
\newblock Computing valid p-values for image segmentation by selective
  inference.
\newblock In {\em Proceedings of the IEEE/CVF Conference on Computer Vision and
  Pattern Recognition}, pages 9553--9562, 2020.

\bibitem{tibshirani2016exact}
R.~J. Tibshirani, J.~Taylor, R.~Lockhart, and R.~Tibshirani.
\newblock Exact post-selection inference for sequential regression procedures.
\newblock {\em Journal of the American Statistical Association},
  111(514):600--620, 2016.

\bibitem{tsukurimichi2022conditional}
T.~Tsukurimichi, Y.~Inatsu, V.~N.~L. Duy, and I.~Takeuchi.
\newblock Conditional selective inference for robust regression and outlier
  detection using piecewise-linear homotopy continuation.
\newblock {\em Annals of the Institute of Statistical Mathematics},
  74(6):1197--1228, 2022.

\bibitem{umezu2017selective}
Y.~Umezu and I.~Takeuchi.
\newblock Selective inference for change point detection in multi-dimensional
  sequences.
\newblock {\em arXiv preprint arXiv:1706.00514}, 2017.

\bibitem{wong2002rule}
W.-K. Wong, A.~Moore, G.~Cooper, and M.~Wagner.
\newblock Rule-based anomaly pattern detection for detecting disease outbreaks.
\newblock In {\em AAAI/IAAI}, pages 217--223, 2002.

\bibitem{yamada2018post}
M.~Yamada, Y.~Umezu, K.~Fukumizu, and I.~Takeuchi.
\newblock Post selection inference with kernels.
\newblock In {\em International conference on artificial intelligence and
  statistics}, pages 152--160. PMLR, 2018.

\bibitem{yang2016selective}
F.~Yang, R.~F. Barber, P.~Jain, and J.~Lafferty.
\newblock Selective inference for group-sparse linear models.
\newblock In {\em Advances in Neural Information Processing Systems}, pages
  2469--2477, 2016.

\bibitem{zaman2001econometric}
A.~Zaman, P.~J. Rousseeuw, and M.~Orhan.
\newblock Econometric applications of high-breakdown robust regression
  techniques.
\newblock {\em Economics Letters}, 71(1):1--8, 2001.

\end{thebibliography}

\newpage
\onecolumn

\section{Appendix} \label{sec:appx}


\subsection{Proof of Lemma \ref{lemma:valid_selective_p}} \label{appx:proof_valid_selective_p}

We have 
\[
\bm \eta_j^\top {\bm X^s \choose \bm X^t} \mid
	\left \{  
	\cO_{\bm X^s, \bm X^t}
	=
	\cO_{\rm obs},
	\cQ_{\bm X^s, \bm X^t}
	=
	\cQ_{\rm obs}
	\right \} 
	\sim 
	{\rm TN} 
	\left (
	\bm \eta_j^\top 
	{\bm \mu^s \choose \bm \mu^t},
	\bm \eta_j^\top \Sigma \bm \eta_j,
	\cZ
	\right ),
\]
which is a truncated normal distribution with a mean $\bm \eta_j^\top {\bm \mu^s \choose \bm \mu^t}$,
 variance $\bm \eta_j^\top \Sigma \bm \eta_j$ in which $
\Sigma = 
\begin{pmatrix}
	\Sigma^s & 0 \\ 
	0 & \Sigma^t
\end{pmatrix}
$, and the truncation region $\cZ$ described in Sec. \ref{subsec:conditional_data_space}.
Therefore, under the null hypothesis,
\begin{align*}
	p_j^{\rm sel}
	\mid 
	\{ \cO_{\bm X^s, \bm X^t}
	=
	\cO_{\rm obs},
	\cQ_{\bm X^s, \bm X^t}
	=
	\cQ_{\rm obs}
	\}
	\sim {\rm Unif}(0, 1).
\end{align*} 
Thus,
$
	\mathbb{P}_{\rm H_{0, j}}  \Big (
	p_j^{\rm sel} \leq \alpha
	\mid 
	\cO_{\bm X^s, \bm X^t}
	=
	\cO_{\rm obs},
	\cQ_{\bm X^s, \bm X^t}
	=
	\cQ_{\rm obs}
	\Big) = \alpha, \forall \alpha \in [0, 1].
$

Next, we have 
\begin{align*}
	&\mathbb{P}_{\rm H_{0, j}}  \Big (p_j^{\rm sel} \leq \alpha \mid \cO_{\bm X^s, \bm X^t} = \cO_{\rm obs} \Big ) \\ 
	&= 
	\int
	\mathbb{P}_{\rm H_{0, j}}  \Big (p_j^{\rm sel} \leq \alpha \mid \cO_{\bm X^s, \bm X^t} = \cO_{\rm obs},  \cQ_{\bm X^s, \bm X^t} = \cQ_{\rm obs} \Big ) ~
	\mathbb{P}_{\rm H_{0, j}}  \Big (\cQ_{\bm X^s, \bm X^t} = \cQ_{\rm obs} \mid \cO_{\bm X^s, \bm X^t} = \cO_{\rm obs} \Big ) d\cQ_{\rm obs} \\ 
	&= 
	\int \alpha 
	~ \mathbb{P}_{\rm H_{0, j}}  \Big (\cQ_{\bm X^s, \bm X^t} = \cQ_{\rm obs} \mid \cO_{\bm X^s, \bm X^t} = \cO_{\rm obs} \Big ) d\cQ_{\rm obs} \\ 
	& = 
	\alpha 
	\int \mathbb{P}_{\rm H_{0, j}}  \Big (\cQ_{\bm X^s, \bm X^t} = \cQ_{\rm obs} \mid \cO_{\bm X^s, \bm X^t} = \cO_{\rm obs} \Big ) d\cQ_{\rm obs} \\ 
	& = 
	\alpha.
\end{align*} 
Finally, we obtain the result in Lemma \ref{lemma:valid_selective_p} as follows:
\begin{align*}
	\mathbb{P}_{\rm H_{0, j}}  \Big (p_j^{\rm sel} \leq \alpha \Big ) 
	& = \sum \limits_{\cO_{\rm obs}}
	\mathbb{P}_{\rm H_{0, j}}  \Big (p_j^{\rm sel} \leq \alpha \mid \cO_{\bm X^s, \bm X^t} = \cO_{\rm obs} \Big ) ~
	\mathbb{P}_{\rm H_{0, j}}  \Big (\cO_{\bm X^s, \bm X^t} = \cO_{\rm obs} \Big ) \\ 
	& = \sum \limits_{\cO_{\rm obs}} \alpha ~ \mathbb{P}_{\rm H_{0, j}}  \Big (\cO_{\bm X^s, \bm X^t} = \cO_{\rm obs} \Big ) \\ 
	& = \alpha \sum \limits_{\cO_{\rm obs}} \mathbb{P}_{\rm H_{0, j}}  \Big (\cO_{\bm X^s, \bm X^t} = \cO_{\rm obs} \Big ) \\ 
	& = \alpha.
\end{align*}


\subsection{Proof of Lemma \ref{lemma:data_line}} \label{appx:proof_lemma_data_line}

According to the second condition in \eq{eq:conditional_data_space}, we have 
\begin{align*}
	\cQ_{\bm X^s, \bm X^t} & =  \cQ_{\rm obs} \\ 
	\Leftrightarrow 
	\left ( 
	I_{n_s + n_t} - 
	\bm b
	\bm \eta_j^\top \right ) 
	{\bm X^s \choose \bm X^t}
	& = 
	\cQ_{\rm obs}\\ 
	\Leftrightarrow 
	{\bm X^s \choose \bm X^t}
	& = 
	\cQ_{\rm obs}
	+ \bm b
	\bm \eta_j^\top  
	{\bm X^s \choose \bm X^t}.
\end{align*}
By defining 
$\bm a = \cQ_{\rm obs}$,
$z = \bm \eta_j^\top {\bm X^s \choose \bm X^t}$, and incorporating the first condition of \eq{eq:conditional_data_space}, we obtain Lemma \ref{lemma:data_line}.


\subsection{Proof of Lemma \ref{lemma:cZ_1}} \label{appx:proof_lemma_cZ_1}

The proof is constructed based on the results presented in \cite{duy2021exact}, in which the authors introduced an approach to characterize the event of OT by using the concept of \emph{parametric linear programming}.
Let us re-written the OT problem between the source and target domain in \eq{eq:ot_problem} as:
\begin{align*}
	\hat{\bm t} = \argmin \limits_{\bm t \in \RR^{n_s n_t}} ~~ & 
	\bm t^\top \bm c(\bm X^s, \bm X^t) \\
	\text{s.t.} ~~
	& H \bm t = \bm h, ~\bm t \geq \bm 0, \nonumber
\end{align*}
where $\bm t = {\rm {vec}}(T)$,
$
	\bm c(\bm X^s, \bm X^t) 
	= {\rm {vec}} (C(\bm X^s, \bm X^t)) 
	= \left [ \Omega {\bm X^s \choose \bm X^t} \right] \circ \left [ \Omega {\bm X^s \choose \bm X^t} \right]
$,
\[
\Omega = {\rm hstack}\left ( 
	I_{n_s} \otimes \bm 1_{n_t}, - \bm 1_{n_s} \otimes I_{n_t} \right ) 
\in \RR^{n_s n_t \times (n_s + n_t)},
\]
${\rm vec}(\cdot)$ is an operator that transforms a matrix into a vector with concatenated rows, the operator $\circ$ is element-wise product, $\rm hstack(\cdot, \cdot)$ is horizontal stack operation, $I_n \in \RR^{n \times n}$ is the identity matrix, and $\bm 1_m \in \RR^m$ is a vector of ones.
The matrix $H$ is defined as 
$ H = \left( H_r ~ H_c \right )^\top \in \RR^{(n_s + n_t) \times n_s n_t}$ in which 
\begin{align*}
	H_r = 
	\begin{bmatrix}
		1 ~ \ldots ~  1 & 0 ~ \ldots ~  0 & \ldots & 0 ~ \ldots ~  0 \\
		0 ~ \ldots ~  0 & 1 ~ \ldots ~  1 & \ldots & 0 ~ \ldots ~  0 \\
		 ~ \ldots ~   &  ~ \ldots ~   & \ldots &  ~ \ldots ~   \\
		0 ~ \ldots ~  0 & 0 ~ \ldots ~  0 & \ldots & 1 ~ \ldots ~  1 \\
	\end{bmatrix} \in \RR^{n_s \times n_s n_t}
\end{align*} 
that performs the sum over the rows of $T$ and 
\begin{align*}
	 H_c = 
	\begin{bmatrix}
		I_{n_t} & I_{n_t} & \ldots & I_{n_t}
	\end{bmatrix} \in \RR^{n_t \times n_s n_t}
\end{align*}
that performs the sum over the columns of $T$, and $\bm h = \left (\bm 1_{n_s}/{n_s} ~ \bm 1_{n_t}/{n_t} \right)^\top \in \RR^{n_s + n_t}$.

Next, we consider the OT problem with the parametrized data $\bm a + \bm b z$:
\begin{align*} 
	& \min \limits_{\bm t \in \RR^{n_s n_t}} ~  
	\bm t^\top 
	[ \Omega (\bm a + \bm b z) 
	\circ
	\Omega (\bm a + \bm b z)]
	~ ~
	\text{s.t.}
	~ ~
	 H \bm t = \bm h, \bm t \geq \bm 0 \\ 
	 \Leftrightarrow
	 & \min \limits_{\bm t \in \RR^{n_s n_t}} ~  
	(\tilde{\bm w} + \tilde{\bm r} z + \tilde{\bm o} z^2)^\top \bm t
	~ ~
	\text{s.t.}
	~ ~
	 H \bm t = \bm h, \bm t \geq \bm 0,
\end{align*}
where 
\[ \tilde{\bm w} = (\Omega \bm a) \circ (\Omega \bm a), 
\quad \tilde{\bm r} = (\Omega \bm a) \circ (\Omega \bm b) + (\Omega \bm b) \circ (\Omega \bm a), 
\quad \text{and} \quad  \tilde{\bm o} = (\Omega \bm b) \circ (\Omega \bm b).\]
By fixing $\cB_u$ as the optimal basic index set of the linear program, the \emph{relative cost vector} w.r.t to the set of non-basic variables $\cB^c_u$ is defines as 
\begin{align*}
	\bm r_{\cB^c_u} = \bm w + \bm r z + \bm o z^2,
\end{align*}
where 
\begin{align} \label{eq:w_r_o}
\bm w = 
\left(
	\tilde{\bm w}_{\cB^c_u}^\top - \tilde{\bm w}_{\cB_u}^\top H_{:, \cB_u}^{-1} H_{:, \cB^c_u}
\right)^\top,
\quad
\bm r = 
\left(
	\tilde{\bm r}_{\cB^c_u}^\top - \tilde{\bm r}_{\cB_u}^\top H_{:, \cB_u}^{-1} H_{:, \cB^c_u}
\right)^\top,
\quad
\bm o = 
\left(
	\tilde{\bm o}_{\cB^c_u}^\top - \tilde{\bm o}_{\cB_u}^\top H_{:, \cB_u}^{-1} H_{:, \cB^c_u}
\right)^\top,
\end{align}
$H_{:, \cB_u}$ is a sub-matrix of $S$ made up of all rows and columns in the set $\cB_u$.
The requirement for $\cB_u$ to be the optimal basis index set is $\bm r_{\cB^c_u} \geq \bm 0$ (i.e., the cost in minimization problem will never decrease when the non-basic variables become positive and enter the basis).
We note that the optimal basis index set $\cB_u$ corresponds to the transportation $\cT_u$.
Therefore, the set $\cZ_u$ can be defined as 
\begin{align*}
	\cZ_u &= \big \{ z \in \RR \mid  \cT_{\bm a + \bm b z} = \cT_u \big\}, \\ 
	&= \big \{ z \in \RR \mid  \cB_{\bm a + \bm b z} = \cB_u \big\}, \\ 
	& = \big \{ z \in \RR \mid  \bm r_{\cB^c_u} = \bm w + \bm r z + \bm o z^2 \geq \bm 0 \big\}.
\end{align*}
Thus, we obtain the result in Lemma \ref{lemma:cZ_1}.

\subsection{Proof of Lemma \ref{lemma:cZ_2}} \label{appx:proof_lemma_cZ_2}

For notational simplicity, let us denote the original data and the data after OT-based DA as follows:
\begin{align*}
	\bm Y = { \bm X^s \choose \bm X^t } \in \RR^{n_s + n_t}, 
	\quad 
	\tilde{\bm Y} = { \hat{\bm X}^s \choose \bm X^t } = \Theta \bm Y, 
	\text{ where }
	\Theta = 
	\begin{pmatrix}
		0_{n_s \times n_s} ~ n_s \hat{T} \\ 
		0_{n_t \times n_s} ~ I_{n_t}
	\end{pmatrix}
	\in \RR^{(n_s + n_t) \times (n_s + n_t)}.
\end{align*}
The procedure of applying MAD on $\tilde{\bm Y}$ is described as follows:

\begin{enumerate}
	\item $\tilde{Y}_{k_1} = {\rm median} (\hat{\bm Y})$. This step can be represented by the following sets:
	\begin{align} \label{eq:condition_1}
	\begin{aligned}
		\cS_v^{1a} = \Big \{ 
			(k, k_1) : \tilde{Y}_k \leq \tilde{Y}_{k_1}
		\Big \},  \quad 
		\cS_v^{1b} = \Big \{ 
			(k, k_1) : \tilde{Y}_k \geq \tilde{Y}_{k_1}
		\Big \},
		\quad k \in [n_s + n_t].
	\end{aligned}
	\end{align}
	\item $|\tilde{Y}_{k_2} - \tilde{Y}_{k_1}| = {\rm median} \Big( \big \{ | \tilde{Y}_k - \tilde{Y}_{k_1} | \big \}_{k \in [n_s + n_t]} \Big )$.
	This step can be represented by the following sets:
	\begin{subequations} \label{eq:condition_2}
	\begin{align} 
		\cS_v^{2a} &= 
		\left \{ 
			s_k : s_k = {\rm sign} \left(\tilde{Y}_k - \tilde{Y}_{k_1}\right)
		\right \}, \\ 
		\cS_v^{2b} &= \Big \{ 
			(k, k_2) : s_k \big (\tilde{Y}_k - \tilde{Y}_{k_1}\big ) \leq s_{k_2} \big (\tilde{Y}_{k_2} - \tilde{Y}_{k_1}\big )
		\Big \}, \\  
		\cS_v^{2c} &= \Big \{ 
			(k, k_2) : s_k \big (\tilde{Y}_k - \tilde{Y}_{k_1}\big ) \geq s_{k_2} \big (\tilde{Y}_{k_2} - \tilde{Y}_{k_1}\big )
		\Big \},
	\end{align}
	\end{subequations}
	for any $k \in [n_s + n_t]$.
	\item Given $\gamma$, $\tilde{Y}_k$ is considered to be an anomaly if 
	$\tilde{Y}_k \not \in \Big [\tilde{Y}_{k_1} - \gamma \big |\tilde{Y}_{k_2} - \tilde{Y}_{k_1} \big |, \tilde{Y}_{k_1} + \gamma \big |\tilde{Y}_{k_2} - \tilde{Y}_{k_1} \big |\Big]$.
	This step can be represented by the following sets:
	\begin{subequations} \label{eq:condition_3}
	\begin{align}
		\cS_v^{3a} &= 
		\Big \{
			k \in [n_s + n_t] :  \tilde{Y}_k < \tilde{Y}_{k_1} - \gamma s_{k_2}\big (\tilde{Y}_{k_2} - \tilde{Y}_{k_1} \big )
		\Big\}, \\ 
		\cS_v^{3b} &= 
		\Big \{
			k \in [n_s + n_t] :  \tilde{Y}_k > \tilde{Y}_{k_1} + \gamma s_{k_2}\big (\tilde{Y}_{k_2} - \tilde{Y}_{k_1} \big )
		\Big\}, \\ 
		\cS_v^{3c} &= 
		\Big \{
			k \in [n_s + n_t] :  
			\tilde{Y}_{k_1} - \gamma s_{k_2}\big (\tilde{Y}_{k_2} - \tilde{Y}_{k_1} \big )
			\leq
			\tilde{Y}_k
			\leq
			\tilde{Y}_{k_1} + \gamma s_{k_2}\big (\tilde{Y}_{k_2} - \tilde{Y}_{k_1} \big ) \label{eq:condition_3_last}
		\Big\}. 
	\end{align}
	\end{subequations}
\end{enumerate}

The entire MAD algorithm can be represented as $\cS_v = \cS_v^{1a} \cup \cS_v^{1b} \cup \cS_v^{2a} \cup \cS_v^{2b} \cup \cS_v^{2c} \cup \cS_v^{3a} \cup \cS_v^{3b} \cup \cS_v^{3c}$.

For any data point $\bm a + \bm b z$, if it satisfies all the inequalities from \eq{eq:condition_1} to \eq{eq:condition_3}, $\cS_{\bm a + \bm b z} = \cS_v$.
Regarding the inequalities in $\cS_v^{1a}$ of \eq{eq:condition_1} w.r.t. $\bm a + \bm b z$, we can write 
\begin{align*}
	&\bm e_k^\top \Theta (\bm a + \bm b z) \leq \bm e_{k_1}^\top \Theta (\bm a + \bm b z) \\ 
	\Leftrightarrow ~
	& ( \bm e_k - \bm e_{k_1} )^\top \Theta \bm b z 
	\leq 
	( \bm e_{k_1} - \bm e_k )^\top \Theta \bm a,
\end{align*}
for all $(k, k_1) \in \cS^{1a}_v$, and $\bm e_k \in \RR^{n_s + n_t}$ is a vector with 1 at the $k^{\rm th}$, and 0 otherwise.
Then we have a system of linear inequalities $\bm p^{1a} z \leq \bm q^{1a}$ where 
\begin{align*}
	\bm p^{1a} = {\rm vector} \left( \Big \{ \bm e_{k, k_1}^\top \Theta \bm b \Big \}_{(k, k_1) \in \cS^{1a}_v} \right), ~
	\bm q^{1a} = {\rm vector} \left( \Big \{ \bm e_{k_1, k}^\top \Theta \bm a \Big \}_{(k, k_1) \in \cS^{1a}_v} \right),
\end{align*}
${\rm vector} (\cdot)$ is the operation that converts a set to a vector, and $\bm e_{k, k_1} = \bm e_k - \bm e_{k_1}$.
%
%
Similarly, we obtain $\bm p^{1b} z \leq \bm q^{1b}$, where
\begin{align*}
	\bm p^{1b} = {\rm vector} \left( \Big \{ \bm e_{k_1, k}^\top \Theta \bm b \Big \}_{(k, k_1) \in \cS^{1a}_v} \right), ~
	\bm q^{1b} = {\rm vector} \left( \Big \{ \bm e_{k, k_1}^\top \Theta \bm a \Big \}_{(k, k_1) \in \cS^{1a}_v} \right).
\end{align*}
In regard to \eq{eq:condition_2}, we have three systems of linear inequalities $\bm p^{2a} z \leq \bm q^{2a}$, $\bm p^{2b} z \leq \bm q^{2b}$, and $\bm p^{2c} z \leq \bm q^{2c}$, where 
\begin{align*}
	\bm p^{2a} & = {\rm vector} \left( \Big \{ s_k \bm e_{k_1, k}^\top \Theta \bm b \Big \}_{s_k \in \cS^{2a}_v} \right), ~
	\bm q^{2a} = {\rm vector} \left( \Big \{ s_k \bm e_{k, k_1}^\top \Theta \bm a \Big \}_{s_k \in \cS^{2a}_v} \right), \\ 
	\bm p^{2b} & = {\rm vector} \left( \Big \{ \big ( s_k \bm e_{k, k_1} - s_{k_2} \bm e_{k_2, k_1} \big )^\top \Theta \bm b \Big \}_{(k, k_2) \in \cS^{2b}_v} \right), ~
	\bm q^{2b} = {\rm vector} \left( \Big \{ \big ( s_{k_2} \bm e_{k_2, k_1} - s_k \bm e_{k, k_1} \big )^\top \Theta \bm a \Big \}_{(k, k_2) \in \cS^{2b}_v} \right), \\ 
	\bm p^{2c} & = {\rm vector} \left( \Big \{ \big (s_{k_2} \bm e_{k_2, k_1} - s_k \bm e_{k, k_1} \big )^\top \Theta \bm b \Big \}_{(k, k_2) \in \cS^{2c}_v} \right), ~
	\bm q^{2c} = {\rm vector} \left( \Big \{ \big (s_k \bm e_{k, k_1} - s_{k_2} \bm e_{k_2, k_1} \big )^\top \Theta \bm a \Big \}_{(k, k_2) \in \cS^{2c}_v} \right).
\end{align*}
With regard to \eq{eq:condition_3}, we have four systems of linear inequalities $\bm p^{3a} z \leq \bm q^{3a}$, $\bm p^{3b} z \leq \bm q^{3b}$, $\bm p^{3c, 1} z \leq \bm q^{3c, 1}$, and $\bm p^{3c, 2} z \leq \bm q^{3c, 2}$ (the \eq{eq:condition_3_last} corresponds to two systems of linear inequalities), where 
\begin{align*}
	\bm p^{3a} & = {\rm vector} \left( \Big \{ \big (\gamma s_{k_2} \bm e_{k_2, k_1} - \bm e_{k_1, k} \big )^\top \Theta \bm b \Big \}_{k \in \cS^{3a}_v} \right), ~
	\bm q^{3a} = {\rm vector} \left( \Big \{ \big (\bm e_{k_1, k} - \gamma s_{k_2} \bm e_{k_2, k_1} \big )^\top \Theta \bm a \Big \}_{k \in \cS^{3a}_v} \right),\\ 
	\bm p^{3b} & = {\rm vector} \left( \Big \{ \big (\gamma s_{k_2} \bm e_{k_2, k_1} - \bm e_{k, k_1} \big )^\top \Theta \bm b \Big \}_{k \in \cS^{3b}_v} \right), ~
	\bm q^{3b} = {\rm vector} \left( \Big \{ \big (\bm e_{k, k_1} - \gamma s_{k_2} \bm e_{k_2, k_1} \big )^\top \Theta \bm a \Big \}_{k \in \cS^{3b}_v} \right), \\
	\bm p^{3c, 1} & = {\rm vector} \left( \Big \{ \big ( \bm e_{k_1, k} - \gamma s_{k_2} \bm e_{k_2, k_1} \big )^\top \Theta \bm b \Big \}_{k \in \cS^{3c}_v} \right), ~
	\bm q^{3c, 1} = {\rm vector} \left( \Big \{ \big (\gamma s_{k_2} \bm e_{k_2, k_1} - \bm e_{k_1, k}\big )^\top \Theta \bm a \Big \}_{k \in \cS^{3c}_v} \right) ,\\
	\bm p^{3c, 2} & = {\rm vector} \left( \Big \{ \big ( \bm e_{k, k_1} - \gamma s_{k_2} \bm e_{k_2, k_1} \big )^\top \Theta \bm b \Big \}_{k \in \cS^{3c}_v} \right), ~
	\bm q^{3c, 2} = {\rm vector} \left( \Big \{ \big (\gamma s_{k_2} \bm e_{k_2, k_1} - \bm e_{k, k_1}\big )^\top \Theta \bm a \Big \}_{k \in \cS^{3c}_v} \right).
\end{align*}
Finally, by defining 
\begin{align}
	\bm p &= {\rm vstack} \Big ( \bm p^{1a}, \bm p^{1b}, \bm p^{2a}, \bm p^{2b}, \bm p^{2c}, \bm p^{3a}, \bm p^{3b}, \bm p^{3c, 1}, \bm p^{3c, 2}\Big ), \\ 
	\text{and} ~~ \bm q &= {\rm vstack} \Big ( \bm q^{1a}, \bm q^{1b}, \bm q^{2a}, \bm q^{2b}, \bm q^{2c}, \bm q^{3a}, \bm q^{3b}, \bm q^{3c, 1}, \bm q^{3c, 2}\Big ),
\end{align}
we obtain the result in Lemma \ref{lemma:cZ_2}.

%
%
%


\subsection{Comparison with \cite{tsukurimichi2022conditional}} \label{appx:comparison_with_tsukurimichi}

Since the method of \cite{tsukurimichi2022conditional} is not applicable to our proposed setting, we have to introduce an extended setting for conducting the experiments of FPR comparison with their method.
The method of \cite{tsukurimichi2022conditional} primarily focus on a regression problem.
Let us consider $(X^s, \bm Y^s)$ and $(X^t, \bm Y^t)$, where $X^s \in \RR^{n_s \times p}$ and $X^t \in \RR^{n_t \times p}$ are given feature matrices assumed to be non-random,
\begin{align*}
	\RR^{n_s} \ni \bm Y^s \sim \NN (\bm \mu^s, \Sigma^s) \quad \text{and} \quad \RR^{n_t} \ni \bm Y^t \sim \NN (\bm \mu^t, \Sigma^t).
\end{align*}
The cost matrix is defined as 
\begin{align*} 
	C(\bm Y^s, \bm Y^t) 
	& = \big[(Y_i^s - Y_j^t)^2 \big]_{ij} \in \RR^{n_s \times n_t}.
\end{align*}
After obtaining $\hat{T}$ by solving the OT problem with $C( \bm Y^s, \bm Y^t)$, we transform the data from the source domain to the target domain and conduct a robust regression in the target domain.
The problem of hypothesis testing is the same as \cite{tsukurimichi2022conditional}.
However, when computing the $p$-value for conducting the test, the method of \cite{tsukurimichi2022conditional} does not take into account the influence of DA. 
Therefore, their method is invalid and could not control the FPR.
In contrast, with the proposed method, we can successfully control the FPR by handling the effect of DA process.
In the experiment for FPR comparison, we set $n_s \in \{20, 40, 60, 80\}$, $n_t = 15$, $p = 5$ and used  LAD (least absolute deviations) as the robust regression method.
The results are shown in Fig. \ref{fig:comparison_with_tsukurimichi}.
The proposed {\tt CAD-DA} could properly control the FPR under $\alpha = 0.05$ whereas the existing method \cite{tsukurimichi2022conditional} failed because it could not account for the influence of DA process.


\subsection{Robustness experimentes} \label{appx:robustness}

We conducted the following experiments:

\begin{itemize}
 \item  Non-normal data: we considered the noise following Laplace distribution, skew normal distribution (skewness coefficient 10) and $t_{20}$ distribution.
We set $n_s \in \{ 100, 150, 200, 250 \}$, $n_t = 45$, and $d = 1$.
Each experiment was repeated 120 times.
We tested the FPR for both $\alpha = 0.05$ and $\alpha = 0.1$. 
The FPR results are shown in Figs. \ref{fig:app_laplace}, \ref{fig:app_skew_normal} and \ref{fig:app_t_20}. 
We confirmed that the {\tt CAD-DA} still maintained good performance on FPR control.

 \item Estimated variance: 
the variances of the noises were estimated from the data by using empirical variance.
Our proposed {\tt CAD-DA} method could properly control the FPR (Fig. \ref{fig:app_unknown_sigma}). 
\end{itemize}

\begin{figure}[H]
\centering
\begin{subfigure}{0.3\linewidth}
\centering
\includegraphics[width=\linewidth]{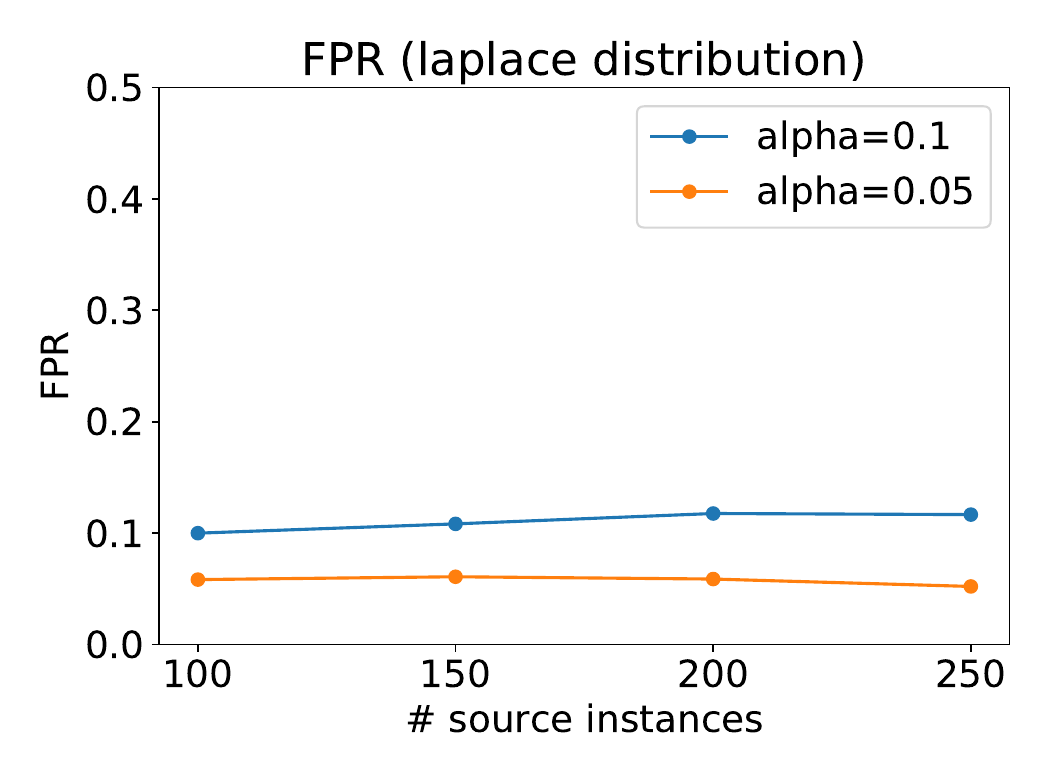}
\caption{Laplace distribution}
\label{fig:app_laplace}
\end{subfigure}
\begin{subfigure}{0.3\linewidth}
\centering
\includegraphics[width=\linewidth]{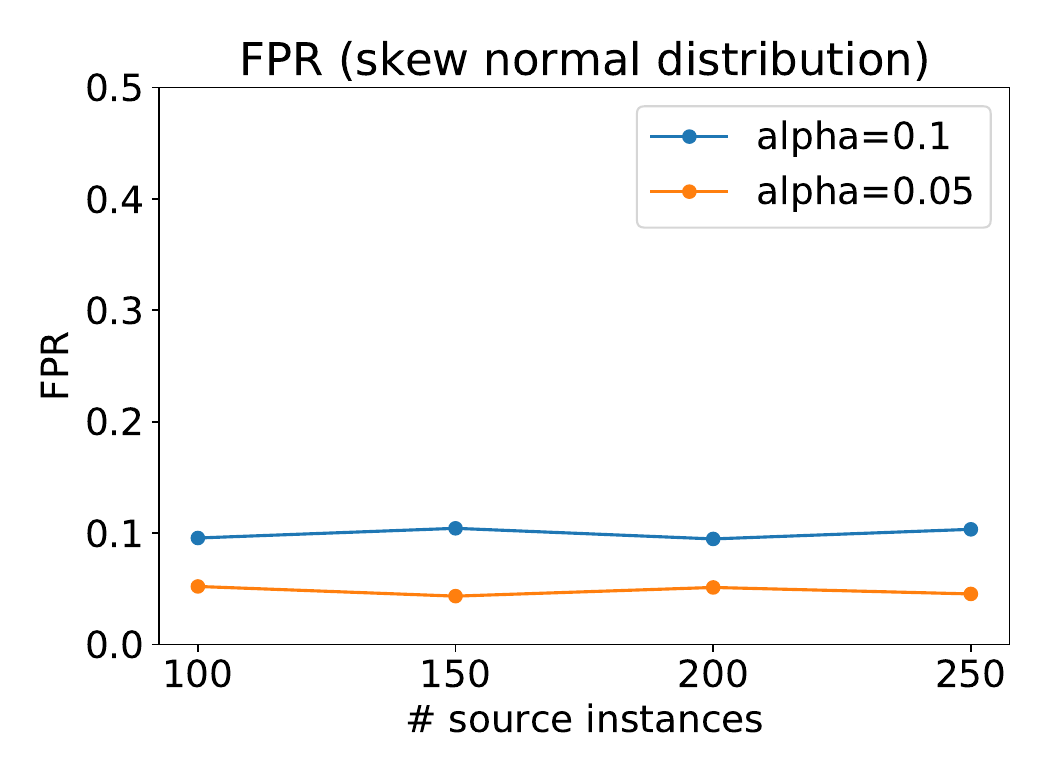}
\caption{Skew normal distribution}
\label{fig:app_skew_normal}
\end{subfigure}

\begin{subfigure}{0.3\linewidth}
\centering
\includegraphics[width=\linewidth]{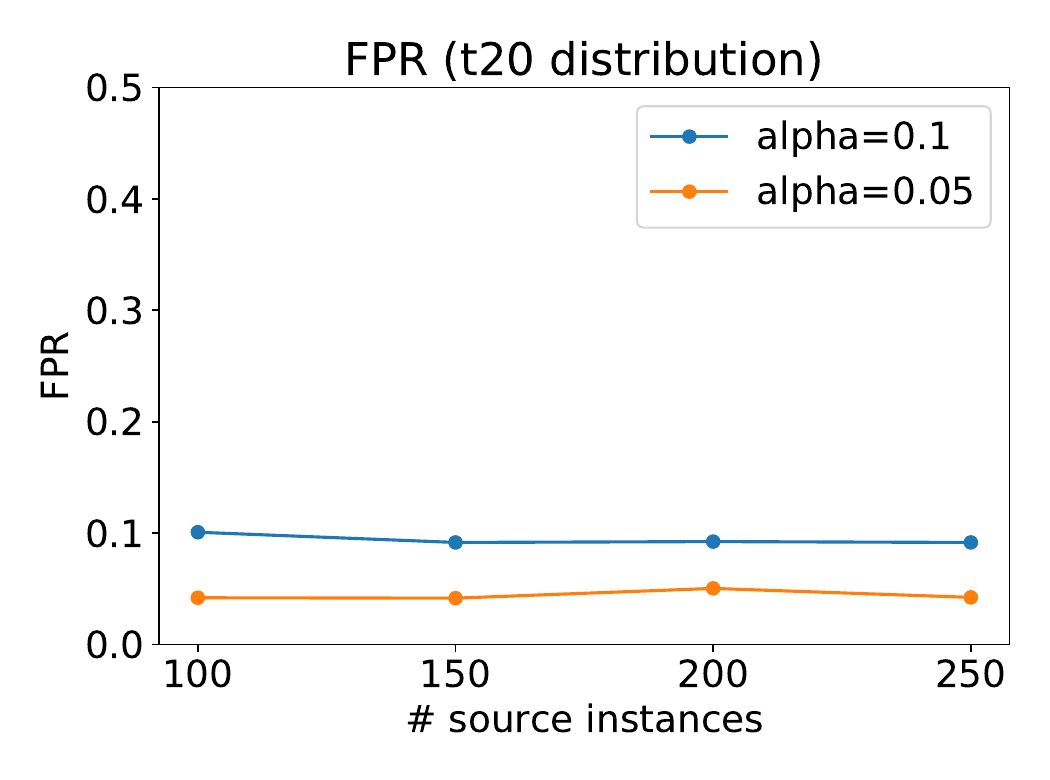}
\caption{$t_{20}$ distribution}
\label{fig:app_t_20}
\end{subfigure}
\begin{subfigure}{0.3\linewidth}
\centering
\includegraphics[width=\linewidth]{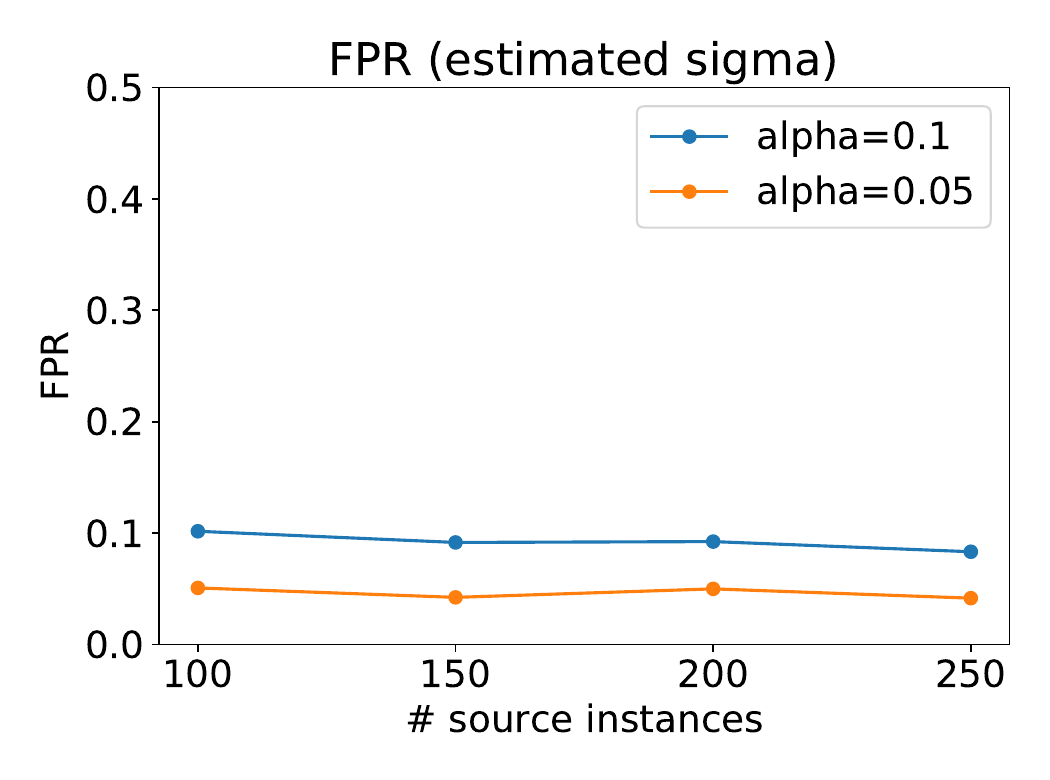}
\caption{Estimated variance}
\label{fig:app_unknown_sigma}
\end{subfigure}
\caption{False positive rate of the proposed {\tt CAD-DA} method when data is non-normal or variance is unknown.}
\label{fig:violate_assumption}
\end{figure}

\end{document}